%% file: main_camera_ready.tex
\documentclass[letterpaper]{article} 
\usepackage{aaai25}  
\usepackage{times}  
\usepackage{helvet}  
\usepackage{courier}  
\usepackage[hyphens]{url}  
\usepackage{graphicx} 
\usepackage{natbib}  
\usepackage{caption} 
\usepackage{makecell}
\usepackage{placeins}

\usepackage{enumitem}
\usepackage[ruled,linesnumbered,noend]{algorithm2e}
\usepackage{subcaption}
\usepackage{multirow}
\usepackage{booktabs}
\usepackage{adjustbox}
\usepackage{mathtools}
\usepackage{microtype}
\usepackage{amsthm}
\usepackage{bigstrut}
\usepackage[bb=dsserif]{mathalpha}

\newcommand{\influenceMinimization}{IMIN\xspace}

\newcommand{\method}{\textsc{DiffIM}\xspace}

\newcommand{\naive}{\textsc{DiffIM}\xspace}
\newcommand{\adv}{\textsc{DiffIM+}\xspace}
\newcommand{\advp}{\textsc{DiffIM++}\xspace}

\newcommand{\NAIVE}{\textsc{DiffIM}\xspace}
\newcommand{\ADV}{\textsc{DiffIM+}\xspace}
\newcommand{\ADVP}{\textsc{DiffIM++}\xspace}

\newcommand{\GNN}{\operatorname{GNN}}

\newcommand{\EL}{\textt{ET}\xspace}
\newcommand{\CL}{\textt{CL}\xspace}
\newcommand{\WL}{\textt{WC}\xspace}

\newcommand{\MDS}{\textsc{MDS}\xspace}
\newcommand{\KED}{\textsc{KED}\xspace}
\newcommand{\BPM}{\textsc{BPM}\xspace}
\newcommand{\Greedy}{\textsc{Greedy}\xspace}
\newcommand{\MBPM}{\textsc{MBPM}\xspace}
\newcommand{\RIS}{\textsc{RIS}\xspace}

\newcommand{\textt}[1]{\scalebox{1.0}{\texttt{#1}}}

\newtheorem{theorem}{Theorem}
\newtheorem{lemma}{Lemma}
\newtheorem{definition}{Definition}
\newtheorem{problem}{Problem}

\newtheorem{remark}{Remark}

\SetAlFnt{\small}

\def\mydefbb#1{\expandafter\def\csname bb#1\endcsname{\ensuremath{\mathbb{#1}}}}
\def\mydefallbb#1{\ifx#1\mydefallbb\else\mydefbb#1\expandafter\mydefallbb\fi}
\mydefallbb ABCDEFGHIJKLMNOPQRSTUVWXYZ\mydefallbb

\def\mydefcal#1{\expandafter\def\csname cal#1\endcsname{\ensuremath{\mathcal{#1}}}}
\def\mydefallcal#1{\ifx#1\mydefallcal\else\mydefcal#1\expandafter\mydefallcal\fi}
\mydefallcal ABCDEFGHIJKLMNOPQRSTUVWXYZ\mydefallcal

\newcommand{\Abs}[1]{\lvert #1 \rvert}
\newcommand{\Set}[1]{\{#1\}}

\newcommand{\citepnop}[1]{\citeauthor{#1} \citeyear{#1}}

\title{\method: Differentiable Influence Minimization with\\Surrogate Modeling and Continuous Relaxation}
\author{
    Junghun Lee, 
    Hyunju Kim, 
    Fanchen Bu, 
    Jihoon Ko, 
    Kijung Shin 
}
\affiliations{
    KAIST, Daejeon, South Korea\\
    \{junghun.lee, hyunju.kim, boqvezen97, jihoonko, kijungs\}@kaist.ac.kr
}

\begin{document}

\maketitle

\begin{abstract}
In social networks, people \textit{influence} each other through social links, which can be represented as propagation among nodes in graphs.
\textit{Influence minimization} (\influenceMinimization) is the problem of manipulating the structures of an input graph (e.g., removing edges) to reduce the propagation among nodes.
\influenceMinimization can represent time-critical real-world applications, such as rumor blocking, but \influenceMinimization is theoretically difficult and computationally expensive.
Moreover, the \textit{discrete} nature of \influenceMinimization hinders the usage of powerful machine learning techniques, which requires \textit{differentiable} computation.
In this work, we propose \method, a novel method for \influenceMinimization with two \textit{differentiable} schemes for acceleration:
(1) \textit{surrogate modeling} for efficient influence estimation, which avoids time-consuming simulations (e.g., Monte Carlo), and
(2) the \textit{continuous relaxation} of decisions, which avoids the evaluation of individual discrete decisions (e.g., removing an edge).
We further propose a third accelerating scheme, \textit{gradient-driven selection}, that chooses edges \textit{instantly based on gradients} without optimization (spec., gradient descent iterations) on each test instance.
Through extensive experiments on real-world graphs, we show that each proposed scheme significantly improves speed with little (or even no) \influenceMinimization performance degradation.
Our method is Pareto-optimal (i.e., no baseline is faster \textit{and} more effective than it) and typically several orders of magnitude (spec., up to 15,160$\times$) faster than the most effective baseline while being more effective. 
\end{abstract}

\begin{links}
     \link{Code, datasets and online appendix}{https://github.com/junghunl/DiffIM}
\end{links}

\section{Introduction}
\label{sec:intro}
In both online and offline social networks, a common phenomenon is \textit{influence}.
That is, people influence other people through social links.
Typical examples include the spread of information (e.g., rumors) and the contagion of a disease (e.g., COVID-19).
We can model social networks as graphs and use the propagation among nodes to simulate such processes~\citep{kempe2005influential}, and
several \textit{diffusion models} mathematically model such propagation.

While the problem of influence maximization has been widely studied, prior research has also explored  \textit{influence minimization} (\influenceMinimization), where one aims to manipulate graph structures to reduce the propagation among nodes.
\influenceMinimization is relevant to real-world scenarios, such as blocking the spread of rumors or diseases, and it has been studied with several different formulations~\citep{yan2019rumor,ni2023misinformation}.
In this work, we mainly focus on a formulation with edge removal under the independent cascade (IC) model~\citep{kempe2003maximizing}, due to its realisticness and generality; and we provide discussions and experiments on other models, spec., the linear threshold (LT) model~\citep{kempe2003maximizing} and the general Markov chain susceptible-infected-recovered (G-SIR) model~\citep{yi2022edge}, in Appendix~B.
Specifically, the IC model is realistic and widely considered for modeling the spread of information~\citep{tripathy2010study} 
and diseases~\citep{borgs2014maximizing}; 
and edge removal is general, including node removal, another widely considered graph manipulation, as a special case.
For such real-world scenarios, we need timely actions since any delay could witness an exponential explosion in the spread of information~\citep{jin2013epidemiological} and diseases~\citep{platto2021history}.

However, this problem is NP-hard, and even simply computing influence under the IC model is computationally expensive (spec., \#P-hard; \citepnop{chen2010scalable}).
Therefore, several ideas, including Monte Carlo (MC) simulation, bond percolation~\citep{Kimura2009blocking}, reverse influence sampling~\cite{borgs2014maximizing,yi2022edge}, and marginal-decrement heuristics~\citep{yan2019rumor}, are available for influence estimation.
However, they offer limited speed improvements (e.g., due to extensive sampling requirements) and/or rely on assumptions (e.g., acyclic input graphs) that do not generally hold in practice.

Notably, existing methods approach the \influenceMinimization problem as a \textit{discrete} optimization problem, which hinders the application of powerful continuous-optimization (e.g., gradient descent) and machine-learning (e.g., neural networks) techniques.

Instead, we propose a novel method for \influenceMinimization, called \method, with two main \textit{differentiable} schemes:
(1) \textit{surrogate modeling} for efficient influence estimation, 
and (2) \textit{continuous relaxation} of edge removal.
To the best of our knowledge, we are the first to approach influence minimization using \textit{differentiable learning} instead of discrete optimization.

\paragraph{Surrogate modeling.}
First, we propose an efficient scheme for influence estimation.
Inspired by~\citet{ko2020monstor}, we propose to train graph neural networks (GNNs) to ``predict'' the influence when given an input graph and seed nodes (i.e., the nodes where the propagation starts).
Such surrogate modeling leverages the efficiency of GNNs and avoids time-consuming MC simulations or other estimation methods.
Although GNN training introduces additional overhead, it is affordable because once trained in advance, the GNN efficiently estimates the influence of unseen graphs and/or seed nodes.

\paragraph{Continuous relaxation.}
Second, to further enhance speed, we propose to relax the edge removal decisions, so that we can directly optimize continuous (i.e., probabilistic) decisions without evaluating individual edge removal.
Specifically, for each edge, instead of considering a binary decision (to remove it or not), we consider a probabilistic decision representing the probability of removing it.
Such relaxation also allows us to incorporate powerful machine-learning techniques, which typically use gradient descent and thus cannot be naively applied to the original discrete problem.

\paragraph{Gradient-driven selection.}
These two schemes enable us to compute gradients w.r.t. the probabilistic decisions on removing edges, and the gradient of each edge is naturally interpreted as its ``sensitivity''.
Specifically, with influence as the objective, the influence is more sensitive to the edges with higher gradients, and removing such edges is expected to reduce the influence more effectively.
Hence, we further propose our third speed-up scheme, \textit{gradient-driven selection}, that removes edges based on their gradients, without further optimization (spec., gradient descent iterations).

\begin{table}[t!]
    \setlength{\tabcolsep}{1mm}
    \centering
    \begin{adjustbox}{max width=\linewidth}
    \begin{tabular}{c|l}
        \toprule
        \textbf{Symbol} & \textbf{Definition} \\
        \midrule
        $G=(V, E)$ & a graph with a node set $V$ and  an edge set $E$ \\
        $p: E \to [0, 1]$ & activation probabilities \\
        $S \subseteq V$ & a seed set \\
        $\pi(v; G, p, S)$ & the influenced probability of $v$ under $\operatorname{IC}(G, p, S)$ \\
        $\sigma(S; G, p)$ & the expected influence of $S$ under $\operatorname{IC}(G, p, S)$ \\
        $\GNN_\theta(v; G, p, S)$ & the predicted influenced probability of $v$ by $\GNN_\theta$ \\
    \bottomrule
    \end{tabular}
    \end{adjustbox}
    \caption{Frequently-used notations}
    \label{tab:notation}
\end{table}

With the three proposed schemes, we propose three versions, \NAIVE, \ADV, and \ADVP, equipped with the first scheme, the first two schemes, and all three schemes, respectively. More schemes result in faster speed, with little (or even no) \influenceMinimization performance degradation.

Through extensive experiments on three real-world graphs, we show the superiority of the \method family over baselines.
Specifically, all the versions of \method are Pareto-optimal (i.e., no baseline is faster \textit{and} more effective than any version of \method), and they are typically orders of magnitude faster (spec., up to 15,160$\times$) than the most effective baseline, while also being more effective.
We also show their ability to perform well when trained and tested on different graphs.

{In short, our main contributions are three-fold:}
\begin{itemize}  
    \item \textbf{Differentiable learning based approach:} 
    To the best of our knowledge, we are the first to tackle \influenceMinimization using differentiable learning, instead of discrete combinatorial optimization.
    We make this approach feasible by leveraging GNNs as surrogate models and employing the continuous relaxation of edge removal.
    \item \textbf{Gradient-driven acceleration:} We propose another speed-up scheme, \textit{gradient-driven selection}, that selects edges based on their gradients without additional test-time optimization (spec., gradient descent iterations).
    \item \textbf{Extensive experiments:} 
    We demonstrate the empirical superiority of our methods over baselines in \influenceMinimization on real-world graphs, in terms of both speed and effectiveness.
\end{itemize}

\section{Preliminaries}
\label{sec:prelim}
\paragraph{Basic concepts.}
Refer to Table~\ref{tab:notation} for frequently-used notations.
Let $\bbN$ be the set of positive integers, and let $[n]$ be $\{1,\dots, n\}$.
A graph $G=(V, E)$ is defined by a node set $V$ and an edge set $E$.
We consider \textit{directed} edges, i.e.,
each edge $e = (u, v) \in E$ is a directed link from $u \in V$ to $v \in V$.

\begin{definition}[Independent cascade (IC) model]\label{def:IC}
    Given 
    (1) a graph $G = (V, E)$,
    (2) activation probabilities $p: E \rightarrow [0,1]$,
    (3) a seed set $S \subseteq V$, the IC model 
    $\operatorname{IC}(G, p, S)$ is a stochastic process defined as follows:
    \begin{itemize} 
        \item \textbf{Initialization:}         
        At time step $t = 0$, each seed node $v_S \in S$ is activated and each non-seed node remains inactive.
        \item \textbf{Diffusion steps:}
        At each step $t \geq 1$,
        each node $v$ that is activated \textit{in the previous step} $t - 1$ 
        activates each of its inactive out-neighbor $u$ with activation probability $p(v, u)$.\footnote{The set of \textit{out-neighbors} of a node $v$ is $\{u\in V: (v, u)\in E\}$.}
        That is, each activated node remains active for the whole process but can only activate other nodes one step after its activation.
        The process terminates when no node is activated in the previous step.
    \end{itemize}
\end{definition}

As shown in Def.~\ref{def:IC}, the process is stochastic and the states of nodes at the termination are thus probabilistic.

\begin{definition}[Influenced probabilities and expected influence]\label{def:act_prob_and_exp_inf}
    Given $G = (V, E)$, $p: E \to [0, 1]$, and $S \subseteq V$,
    for each node $v \in V$,
    the \textit{influenced probability} of $v$, denoted by $\pi(v; G, p, S)$, is the probability of $v$ being influenced (i.e., active) when the process of $\operatorname{IC}(G, p, S)$ terminates,
    and the expected influence of the seed set $S$,
    denoted by $\sigma(S; G, p)$,
    is defined as the expected number of finally influenced nodes, i.e., $\sigma(S; G, p) \coloneqq \sum_{v \in V} \pi(v; G, p, S)$.
\end{definition}

\paragraph{Graph neural networks (GNNs).}
In GNNs, there are two main types of operators on the node features: feature transformation and propagation~\citep{zhu2021interpreting}.
A feature transformation operator transforms node features at each layer into the next layer via nonlinear transformation.
A propagation operator passes the features of a node to its neighbors,
and updates the feature of each node by aggregating the features of its neighbors,
typically in the form of $H \gets \Tilde{A}H$, where 
$H$ is a node feature matrix, and
$\Tilde{A}$ is a normalized adjacency matrix.
The strength of different edges (i.e., edge weights or edge probabilities) can be incorporated in the entries of $\Tilde{A}$.

\section{Related Work}
\label{sec:related}
\paragraph{Influence estimation.}
As mentioned in Sec.~\ref{sec:intro},
exactly computing the influence under the IC model is costly, and thus
several methods have been proposed for influence estimation.
The most related approach is by~\citet{ko2020monstor}, where graph neural networks (GNNs) are used to learn the influence under the IC model.
The Monte Carlo (MC) simulation (i.e., taking the mean value of samplings) has been a common practice for influence computation~\citep{zhou2013ublf,yang2019influence,manouchehri2021temporal}.
\citet{Kimura2009blocking} used the bond percolation method, and \citet{yi2022edge} used the reverse influence sampling instead of MC, but samplings are still required.

\paragraph{Influence minimization.}
Influence minimization (\influenceMinimization) has been widely studied, and it has several different variants.
Researchers have considered node removal~\citep{zhu2021misinformation, ni2023misinformation}
and edge removal~\citep{Kimura2009blocking, tong2012gelling};
and different models other than IC~\citep{dai2022minimizing} have also been considered (note that we also considered other models; see Appendix~B).
Moreover, prior studies have explored blocking propagation to specific targets~\citep{jiang2022rumordecay,wang2020efficient} and/or without specific seed sets~\citep{zareie2022rumour}, as well as
active defense by propagating opposite information~\citep{budak2011limiting, luo2014time}.
In this work, we consider edge removal under the IC model. While
\citet{yan2019rumor} used the same problem formulation, their analysis, and proposed method were limited to acyclic graphs.
Due to the difficulty of influence computation or even estimation, different heuristics without direct influence estimation have also been considered.
They proposed considering the incremental differences when removing each edge, assuming that the input graph is acyclic.
\citet{tong2012gelling} proposed to choose the edges according to the leading eigenvalues of the adjacency matrix, but the effect of seed nodes is not considered in the method.
See Appendix~D 
for more details on existing \influenceMinimization methods.
Note that no existing method for \influenceMinimization has considered a differentiable learning scheme, which is a novel approach introduced in this work.

\section{Problem Statement and Hardness}
\label{sec:pro_state_hard}

As mentioned in Sec.~\ref{sec:related}, there are different problem formulations for influence minimization (\influenceMinimization).
One can consider different graph manipulation (edge removal or node removal) and different diffusion models.
In this work, we consider the formulation with \textit{edge removal} under the \textit{independent cascade} (IC) model (see Sec.~\ref{sec:prelim}) due to the following reasons:
\begin{itemize}  
    \item The formulation with edge removal is more \textit{general} than that with node removal.
    Specifically, node removal can be seen as edge removal with additional constraints that the edges incident to a node should be all kept or all removed.
    \item Blocking spread between users (i.e., edge removal), such as through contact restriction, is often more feasible than completely removing a user (i.e., node removal).
    \item Regarding the diffusion model, the IC model has been widely considered for the spread of information (e.g., rumors)~\citep{tripathy2010study,xu2015scalable,shelke2019source}
    due to its simple yet realistic nature.   
    However, note that \textbf{our proposed approach is not limited to the IC model} but can be applied to more diffusion models, spec., the linear threshold (LT) model~\citep{kempe2003maximizing} and the general Markov chain susceptible-infected-recovered (G-SIR) model~\citep{yi2022edge}, as explored in Appendix~B. 
\end{itemize}
Hereafter, we simply call the considered problem \textit{influence minimization} when no confusion is likely.
\begin{problem}[influence minimization]\label{problem:rumor_blocking} \
\begin{itemize}
    \item \textbf{Given:} 
    a graph $G = (V, E)$, 
    activation probabilities $p: E \to [0, 1]$, 
    a seed set $S \subseteq V$, 
    and a budget $b \in \bbN$,
    \item \textbf{Find:} a set $\calE$ of $b$ edges, i.e., $\calE \subseteq E$ and $|\calE| = b$,
    \item \textbf{to Minimize:} the expected influence of $S$ after removing the edges in $\calE$ from G, i.e., $\sigma(S; G_{\setminus\mathcal{E}}, p_{\setminus\mathcal{E}})$ with $G_{\setminus\mathcal{E}} \coloneqq (V, E\setminus \mathcal{E})$ and
    $p_{\setminus\mathcal{E}}(e) = p(e), \forall e \in E\setminus \mathcal{E}$.
\end{itemize}
\end{problem}

We show the NP-hardness of influence minimization (see Appendix~A.1), 
and \citet{yan2019rumor} proved that influence minimization is non-submodular.
\begin{theorem}\label{thm:np_hard}
    Influence minimization (Problem~\ref{problem:rumor_blocking}) is NP-hard.
\end{theorem}

\begin{theorem}[\citepnop{yan2019rumor}] \label{thm:non_subm}
    Influence minimization is non-submodular, i.e., $f(\calE; G,$ $p, S) \coloneqq 
\sigma(S; G, p) - 
\sigma(S; G_{\setminus\mathcal{E}}, p_{\setminus\mathcal{E}})$ is not submodular w.r.t. $\calE$. 
\end{theorem}

\section{Proposed Method}
\label{sec:method}

Thms.~\ref{thm:np_hard}-\ref{thm:non_subm} show the non-triviality of influence minimization (\influenceMinimization; Problem~\ref{problem:rumor_blocking}).
Below, we further analyze the challenges in \influenceMinimization and propose our method, \method, to address them.

\subsection{Naive algorithms and their problems}\label{subsec:method:naive_greedy}
Before introducing our method \method, we {discuss} some naive algorithms,
and analyze their problems.

A naive enumeration algorithm evaluates all possible combinations of $b$ edges and chooses the best combination.
This requires computing influence for $\binom{\Abs{E}}{b} = \Theta(\Abs{E}^{b})$ times, which is computationally prohibitive.
One can reduce the frequency of influence computation by adapting it in an incremental manner.
There are $b$ rounds in total, and in each round, we choose an edge whose removal reduces the expected influence of $S$ most.
See Alg.~2 
in Appendix~C 
for the pseudo-code of such an incremental greedy algorithm.
Although such an idea requires computing influence for only $O(b\Abs{E})$ times, as discussed in Secs.~\ref{sec:intro} \& \ref{sec:related}, the exact computation of expected influence is computationally prohibitive,
and the existing estimation methods, e.g., Monte Carlo (MC) simulation, are still time-consuming because they require extensive sampling.
Below, we shall propose multiple schemes to speed up the process of choosing each edge.

\subsection{\NAIVE: Surrogate modeling for efficient influence estimation without simulation}
\label{subsec:method:ours_greedy}
We shall first address the problem of time-consuming estimation of influence.
We propose to use graph neural networks (GNNs), which are computationally efficient.

The high-level idea is to use a GNN as a {surrogate model} (neural approximation) of the influenced probabilities $\pi(\cdot)$ (see Def.~\ref{def:act_prob_and_exp_inf}).
That is, we see $\pi$ as a black-box function,
and we aim to train a $\GNN_{\theta}$ parameterized by $\theta$ such that
$\GNN_{\theta}(v; G, p, S) \approx \pi(v; G, p, S), \forall G, p, S, v$.
{Specifically, $p$ is used as edge weights and $S$ is represented by one-dimensional binary node features.}
Although GNN training introduces additional overhead, it is affordable since it can be done once in advance on existing {or randomly generated} data.
Once trained, $\GNN_\theta$ efficiently estimates the influenced probabilities $\GNN_{\theta}(v'; G', p', S') \approx \pi(v'; G', p', S')$ of new unseen cases (see Sec.~\ref{sec:anal} for complexity analysis), avoiding time-consuming MC simulation or other estimation methods (see Sec.~\ref{sec:related} for examples).

Notably, even with a single input graph, we are able to generate multiple data points by generating different seed sets,
and the training process can be easily extended to multiple graphs.
We first obtain the influence of each training seed set by MC simulation 
and use it as the ``ground-truth'' influence, 
and then update the parameters of the GNN w.r.t. the difference between the predicted influence by the GNN and the ``ground-truth'' {influence.}
Specifically, for each seed set $S_i$, the L2-loss is used, i.e.,
\begin{multline*}
    \mathcal{L}(\GNN_\theta; \Tilde{\pi}, \Set{S_i}, G, p) \\ \coloneqq \sqrt{\sum\nolimits_{v \in V} (\GNN_\theta(v; G, p, S_i) - \Tilde{\pi}(v; G, p, S_i))^2}.
\end{multline*}
The final loss function is averaged over seed sets, i.e.,
\begin{equation*}\label{eq:gnn}
   \resizebox{1.0\hsize}{!}{%
    $\mathcal{L}(\GNN_\theta; \Tilde{\pi}, \calS, G, p) \coloneqq \frac{1}{\Abs{\calS}} {\sum_{S_i \in \calS}{\mathcal{L}(\GNN_\theta; \Tilde{\pi}, \Set{S_i}, G, p)}}$.
  }
\end{equation*}
After training, we can use the trained GNN as a {surrogate model} for each test instance $\calT = (G, p, S, b)$ (see Problem~\ref{problem:rumor_blocking}), while still following an incremental greedy scheme, i.e., choosing the edges one by one with the highest effect estimated by the trained GNN, which results in \NAIVE.
See Alg.~\ref{algo:adv} for pseudo-code of \naive, and see
Alg.~3
in Appendix~C
for that of GNN training.

\subsection{\ADV: Continuous relaxation of edge removal without individual removal evaluation}\label{subsec:method:ours_optim}
Although \naive accelerates influence estimation using GNNs as a surrogate model, \naive still needs to evaluate each individual edge removal (spec., compute the estimated influence when each edge is removed), which can still take considerable time even with efficient influence estimation.
To this end, we propose to use \textit{continuous relaxation} of edge removal to avoid individual edge-removal evaluation.

\paragraph{Continuous relaxation.}
The high-level idea is that, for each edge $e = (v, u)$, instead of a binary decision $r(v, u) = \mathbb{1}((v, u) \notin \calE) \in \Set{0, 1}$ (i.e., to keep it or not; recall $\calE$ is the set of edges to be removed), we consider a \textit{probabilistic} decision
$\Tilde{r}(v, u) = \Pr[(v, u) \notin \calE] \in [0, 1]$, i.e., the probability of keeping $e$.
Such relaxation can be readily incorporated into our surrogate-model-based influence estimation.
Specifically, by the multiplication rule:
$\Pr[\text{$v$ activates $u$ via $(v, u)$}] = 
\Pr[\text{$v$ activates $u$} \land \text{$(v, u)$ is kept}] = p(v, u) \tilde{r}(v, u)$.

\begin{algorithm}[t]
    \caption{\NAIVE~/ \ADV / \ADVP}\label{algo:adv}
    \SetKwInput{KwInput}{Input}
    \SetKwInput{KwOutput}{Output}    
    \KwInput{(1) $G =  (V, E)$: an input graph \\
    \quad\quad\quad (2) $p: E \to [0, 1]$: activation probabilities \\
    \quad\quad\quad (3) $S \subseteq V$: a seed set \\
    \quad\quad\quad (4) $b$: an edge-removal budget \\
    \quad\quad\quad (5) $\GNN_{\theta}$: a trained GNN \\
    \quad\quad\quad (6) $\Tilde{r}: E \to [0, 1]$: initial probabilistic decisions

    \hspace*{0pt}\hfill $\triangleright$\ For \ADV and \ADVP \\
    \quad\quad\quad (7) $n_{ep}$: the number of epochs for each removal 
    
    \hspace*{0pt}\hfill $\triangleright$\  For \ADV only \\
    }
    \KwOutput{$\mathcal{E} \subseteq E$: a set of edges chosen to be removed}    
    $\calT \leftarrow (G, p, S, b)$ \hspace*{0pt}\hfill $\triangleright$\ {Initialize the \influenceMinimization problem instance $\calT$}\\
    $\mathcal{E} \leftarrow \emptyset$ \hspace*{0pt}\hfill $\triangleright$\ {Initialize the set of edges to be removed} \\
    \For{$i = 1, 2, \ldots, b$}{        
        $e=$ \texttt{EdgeSelection()} or \texttt{EdgeSelection+()} or \texttt{EdgeSelection++()}
        \hspace*{0pt}\hfill $\triangleright$\ {Select an edge} \\
        $E \leftarrow E\setminus \{e\}$;        $\mathcal{E}\leftarrow\mathcal{E}\cup\{e\}$ \hspace*{0pt}\hfill $\triangleright$\ {Remove the edge} \\
        $G \leftarrow (V, E)$;
        $\calT \leftarrow (G, p, S, b)$
        \hspace*{0pt}\hfill $\triangleright$\ {Update the graph}
        \\
    }
    \Return $\mathcal{E}$
    \label{algo:adv:end}
    
    \vspace{1mm}
    
    \SetKwFunction{FMain}{EdgeSelection}
    \SetKwProg{Fn}{Function}{:}{}
    \Fn{\FMain{}}{
        \Return $\arg\min_{e\in E } \sum_{v \in V} \GNN_{\theta}(v; G, p_{\setminus \calE}, S)$\label{algo:naive:find} \\ 
    } 
    
    \vspace{1mm}
    
    \SetKwFunction{FMainp}{EdgeSelection+}
    \SetKwProg{Fn}{Function}{:}{}
    \Fn{\FMainp{}}{
        \For{$j = 1, 2, \ldots, n_{ep}$}{
            Compute the derivative $\nabla_{\Tilde{r}} \calL := \frac{\partial \calL_{O}(\Tilde{r}; \calT, \GNN_{\theta})}{\partial \Tilde{r}}$ \label{algo:adv:compute} \\
            Update $\Tilde{r}$ via gradient descent {w.r.t. $\nabla_{\Tilde{r}} \calL$}
            \label{algo:adv:update}\\
        }
        \Return $\arg\min_{e\in E} \tilde{r}(e)$ \\
    }
    
    \vspace{1mm}
    
    \SetKwFunction{FMainpp}{EdgeSelection++}
    \SetKwProg{Fn}{Function}{:}{}
    \Fn{\FMainpp{}}{
        $\nabla_{\tilde{r}}(e)=\frac{\partial \sum_{v \in V} \GNN_{\theta}(v; G, p_{\setminus \calE}, S)}{\partial \tilde{r}(e)}, \forall e \in E$\\
        \Return $\arg\max_{e\in E} \nabla_{\tilde{r}}(e)$\\
    }
\end{algorithm}

Therefore, given 
activation probabilities $p: E \to [0, 1]$ and
probabilistic decisions $\tilde{r}: E \to [0, 1]$ on the edges,
we obtain the \textit{modified activation probabilities} 
$\Tilde{p}_{\Tilde{r}}: E \to [0, 1]$
by
$\tilde{p}_{\Tilde{r}}(v, u) \coloneqq p(v, u) \Tilde{r}(v, u), \forall (v, u) \in E$.
Specifically, given 
a graph $G = (V, E)$, 
activation probabilities $p$,
a seed set $S \subseteq V$,
a trained $\GNN_{\theta}$, 
and probabilistic edge-removal decisions $\Tilde{r}$,
the influenced probability of each node $v \in V$ is estimated as
$\GNN_{\theta}(v; G, \Tilde{p}_{\Tilde{r}}, S)$.
Notably, such an estimated influenced probability is differentiable w.r.t. $\Tilde{r}$.

\begin{lemma}\label{lem:gnn_diff}
    $\GNN_{\theta}(v; G, \Tilde{p}_{\Tilde{r}}, S)$ is differentiable w.r.t $\Tilde{r}$.
    
    \noindent \textit{Proof.} \normalfont{See Appendix~A.2.} 
    \qed
\end{lemma}

\paragraph{Differentiable optimization.}
After training $\GNN_{\theta}$, by Lemma~\ref{lem:gnn_diff}, we can now conduct \textit{differentiable} optimization on the probabilistic edge-removal decisions for each test instance $\calT = (G, p, S, b)$ of the \influenceMinimization problem (see Problem~\ref{problem:rumor_blocking}), which results in \ADV (see Alg.~\ref{algo:adv}).
The high-level process is as follows: we fix $\GNN_\theta$ after training, and only update $\Tilde{r}$ to minimize the loss $\calL_{O}$, which has three parts:
\begin{equation}\label{eq:loss_optimize}
    \calL_{O} := \mathcal{L}_\text{obj}
    +\alpha \calL_\text{budget}
    +\beta \calL_\text{certainty},
\end{equation}
where the loss coefficients, $\alpha$ and $\beta$, are hyperparameters.

The first part $\calL_{\text{obj}}$ is regarding the main optimization objective of \influenceMinimization (Problem~\ref{problem:rumor_blocking}), defined as
\begin{multline*}
    \calL_{\text{obj}}(\Tilde{r}; \calT, \GNN_\theta) \\
    := \textstyle{-\frac{\sum_{v \in V} (\GNN_\theta(v; G, p, S) - \GNN_\theta(v; G, \Tilde{p}_{\Tilde{r}}, S))}
    {\sum_{v \in V} \GNN_\theta(v; G, p, S) -\Abs{S}}},
\end{multline*}
which is the \textit{reduction ratio} in the estimated number of influenced non-seed nodes.
Specifically, the numerator is the estimated influence reduction after the probabilistic edge removal $\Tilde{r}$ is applied, and the denominator is the estimated number of influenced non-seed nodes before the edge removal.
When the surrogate model is perfectly accurate, minimizing $\calL_{\text{obj}}$ is equivalent to optimizing the objective of \influenceMinimization.

The second part $\calL_{\text{budget}}$ is regarding the budget constraint:
\begin{equation*}
\calL_{\text{budget}}(\Tilde{r}; \calT, \GNN_\theta) := ({\Abs{E} - \sum\nolimits_{e\in E} \Tilde{r}(e)-b})^2,
\end{equation*}
which is the squared difference between the expected number of removed {edges and the required budget.}
When the budget is exactly used, $\calL_{\text{obj}}$ is $0$, i.e., minimized.

The third part $\calL_{\text{certainty}}$ is regarding the certainty (i.e., closeness to binary) of the probabilistic decisions $\Tilde{r}$:
\begin{multline*}
    \mathcal{L}_\text{certainty}(\Tilde{r}; \calT, \GNN_\theta) \\ := \textstyle{\frac{\sum_{e\in E} (\Tilde{r}(e)\log r(e) - (1-\Tilde{r}(e))\log (1-\Tilde{r}(e)))}{\Abs{E}},}
\end{multline*}
which is inspired by the Shannon entropy,
and $\mathcal{L}_\text{certainty}$ is smaller when {each $\Tilde{r}(e)$ is closer to $0$ or $1$.}

Even with continuous relaxation, optimizing $\calL_O$ remains closely aligned with the \influenceMinimization problem, as discussed in Appendix~A.3.
In Alg.~\ref{algo:adv}, we show pseudo-code of \adv, where some details (e.g., how we initialize and normalize $\tilde{r}$) are omitted and will be deferred to Sec.~\ref{sec:experiments} when we describe the detailed experimental settings.
Given initial probabilistic decisions $\Tilde{r}$, in each iteration, we update it via gradient descent according to the loss function $\Tilde{L}_O$ (Eq.~\eqref{eq:loss_optimize}) and its derivative.
For each specified number $n_{ep}$ of epochs, the edge with the smallest $\tilde{r}(e)$ value is removed.

\subsection{\ADVP: Gradient-driven selection without test-time gradient-descent optimization}\label{subsec:method:ours_instant}
The previous two schemes entail differentiability, enabling us to compute gradients w.r.t. the probabilistic decisions on edges.
The gradient of each edge can be naturally interpreted as its ``sensitivity''.
Specifically, the influence is more sensitive to the edges with higher gradients, and removing such edges is expected to reduce the influence more effectively.
Hence, we propose \ADVP with \textit{gradient-driven edge selection}, which \textit{instantly} removes the edge with the largest gradient in each round, instead of performing optimization (spec., gradient descent) over many epochs.

Alg.~\ref{algo:adv} shows pseudo-code of \advp.
After training $\GNN_{\theta}$, for each test instance $\calT = (G = (V, E), p, S, b)$ of the \influenceMinimization problem,
we compute the derivatives on all edges,
\begin{center}
    $\nabla_{\tilde{r}}(e; \calT) \coloneqq \frac{\partial \sum_{v \in V} \GNN_\theta(v; G, p, S)}{\partial \Tilde{r}(e)}, \forall e \in E$,
\end{center}
and remove the edge with the largest $\nabla_{\tilde{r}}(e; \calT)$ in each round.

\paragraph{Discussion.}
For \ADV, instead of removing edges one by one after $n_{ep}$ epochs, one can choose the bottom-$b$ edges with the lowest $\Tilde{r}$ values at once. 
Similarly, for \ADVP, one can choose the top-$b$ edges with the largest gradient at once.
Empirically, we observe that choosing edges at once in such a way achieves similar or worse performance. 
See Appendix~G.3 
for more detailed results and discussions.

\subsection{Time and space complexities}\label{sec:anal} 
Given a test instance $\calT = (G = (V,E), p, S, b)$ of the \influenceMinimization problem,
assume that 
(1) we use graph convolutional networks (GCNs; \citepnop{kipf2016semi}) with a constant number of layers, as in our experiments, 
(2) the dimensions of hidden features are fixed as constant, which are indeed fixed and much smaller than the size of graphs in our experiments, and
(3) the input graph is sparse (specifically, $\Abs{E} = \Theta(\Abs{V})$), which is indeed true for the datasets in our experiments (see Table~\ref{tab:data}), 
we can derive the time and space complexity of \method based on existing results~\citep{chiang2019cluster,blakely2021time}.

\paragraph{\naive.}
A forward pass of GCN takes $O(\Abs{E})$ time.
In each round, we conduct $O(\Abs{E})$ forward passes, with $b$ rounds in total.
Thus, the time complexity of \naive is $O(b\Abs{E}^2)$.

\paragraph{\adv.}
A backward pass of GCN takes $O(\Abs{E})$ time,
and there are $b$ rounds each with $n_{ep}$ epochs (see Alg.~\ref{algo:adv}).
Therefore, the time complexity of \adv is $O(n_{ep}b\Abs{E})$.

\paragraph{\advp.}
We conduct one backward pass in each of the $b$ rounds, so
the time complexity of \advp is $O(b\Abs{E})$.

\paragraph{Note.}
Regarding the time complexity, 
\naive $>$ \adv $>$ \advp,
as intended.

\paragraph{Space complexity.}
For each version, it is 
$O(\Abs{E} + \Abs{V}) = O(\Abs{E})$, 
dominated by the space complexity of GCN.

\section{Experiments}
\label{sec:experiments}

\begin{table}[t]
    \centering
    \setlength{\tabcolsep}{1mm}
    \begin{tabular}{l|l|r|r|r|r}
        \hline
        \multirow{2}{*}{{dataset}} &\multirow{2}{*}{{abbr.}} &  \multicolumn{2}{c|}{{training graph}} & \multicolumn{2}{c}{{test graph}} \bigstrut \\
        \cline{3-6}        &&\textbf{$|V|$}&\textbf{$|E|$}&\textbf{$|V|$}&\textbf{$|E|$} \bigstrut\\
        \hline
        WannaCry & \WL & $16,246$ & $84,217$ & $19,381$ & $85,202$ \bigstrut[t]\\
        Celebrity & \CL & $7,848$ & $28,839$ & $7,336$ & $27,699$\\ 
        Extended & \EL & $5,636$ & $31,826$ & $5,413$ & $27,146$ \bigstrut[b]\\
        \hline
    \end{tabular}
    \caption{Basic statistics of the real-world datasets.}
    \label{tab:data}
\end{table}

We performed experiments on real-world graphs, 
aiming to answer the following questions:
\begin{itemize}  
    \item \textbf{Q1. Performance:} {How effectively and quickly does \method minimize influence?} 
    \item \textbf{Q2. Scalability:} How does the running time of \method grow with the budget $b$?
    \item \textbf{Q3. Influence estimation quality:} How well does the surrogate GNN model in \method estimate the influence?
    \item \textbf{Q4. Inductivity:} How does \method perform when trained and tested on different/same graphs? 
    \item \textbf{Q5. Ablation studies:} How does each component or algorithmic design affect the performance of \method?
\end{itemize}

\begin{figure*}[t!]
    \centering
    \includegraphics[width=0.9\linewidth]{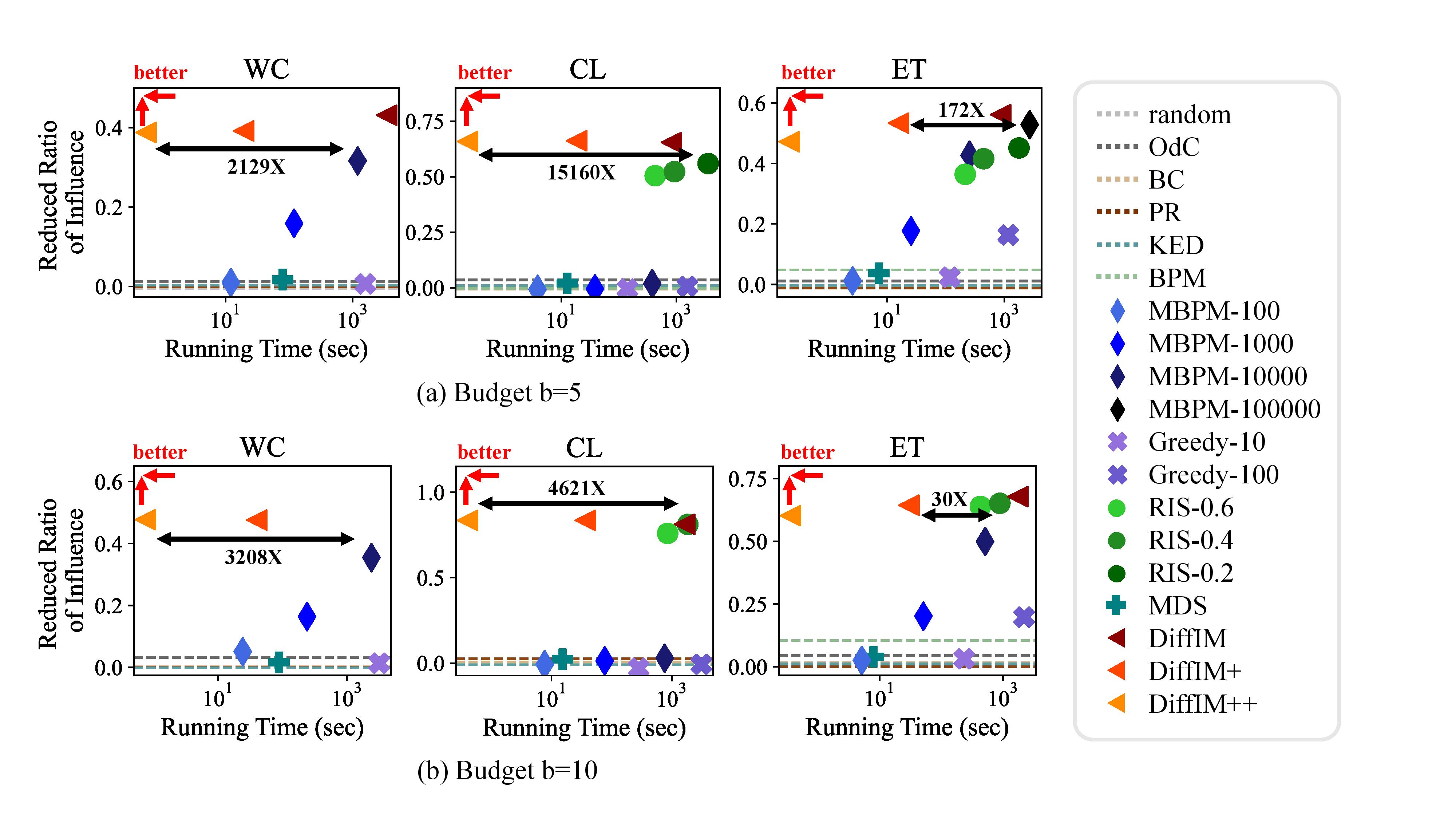}
    \caption{The effectiveness (the reduced ratio of influence) and running time of each method, with budget $b=5$ (top) and $b=10$ (bottom). 
    The baselines with outputs independent of seed sets were represented as horizontal lines.
    We compared the running time of the best baseline to one of our methods with the most similar reduced ratio, and in all cases, \method was
    30 $\times$ to 15,160 $\times$ faster.
    See  Appendix~G.1 for full results with standard deviations.}
    \label{fig:perf_6layer}
\end{figure*}

\subsection{Experimental settings}\label{subsec:exp:setting}

\paragraph{Datasets.} We used three real-world social-network datasets from~\citet{ko2020monstor}: \WL, \CL, and \EL, consisting of interactions logs, e.g., retweets among users~\citep{sabottke2015vulnerability}.
Such interactions naturally represent people's impact on others, i.e., influence.
We provided their basic statistics in Table~\ref{tab:data}.
Each dataset was split into training and test graphs based on a time threshold $t_{th}$: edges before $t_{th}$ were used for training, and those after were used for testing.

\paragraph{Baselines.} We considered the following baselines: 
\begin{enumerate}[align=left,leftmargin=*,topsep=0pt,itemsep=0pt]
    \item \textbf{\textsc{Random}} removes $b$ edges chosen uniformly at random.
    \item  \textbf{\textsc{OdC}} (Out-degree Centrality; \citepnop{kempe2003maximizing})
    removes the top-$b$ edges w.r.t. the sum of out-degrees of their two endpoints.\footnote{{The out-degree of each node $v \in V$ is $\lvert\{u\in V: (v, u)\in E\}\rvert$.}}
    \item \textbf{BC} (Betweenness Centrality; \citepnop{schneider2011suppressing})
    removes the top-$b$ edges w.r.t edge betweenness.
    \item \textbf{PR} (PageRank; \citepnop{page1998pagerank})
    removes the top-$b$ edges w.r.t. the sum of the PageRank scores of their endpoints.
    \item \textbf{\KED}~\cite{tong2012gelling} removes the $b$ edges {to minimize} the leading eigenvalue of the adjacency matrix.
    \item \textbf{\MDS}~\cite{yan2019rumor} greedily removes $b$ edges w.r.t. the importance scores estimated by influenced probability and rumor-spread ability of their endpoints.
    \item \textbf{BPM}~\cite{Kimura2009blocking} 
    removes the top-$b$ edges w.r.t. importance scores estimated using the bond percolation method (BPM).
    BPM does not consider specific seed nodes.
    \item \textbf{Modified \BPM (\MBPM)} is a modified version of \BPM that considers specific seed nodes.
    \item \textbf{\Greedy} removes $b$ edges greedily with influence estimated by Monte Carlo (MC) simulation.
    \item \textbf{\RIS}~\cite{yi2022edge} estimates the importance of each edge for propagating to other nodes using sampling, and removes the $b$ most important edges.
\end{enumerate}
For MBPM and \Greedy, a number after their names denotes the number of samplings (e.g., \Greedy-100 denotes \Greedy with $100$ samplings). 
For \RIS, a small $\epsilon$ increases sampling numbers and improves edge importance estimation accuracy.
Note that all these baselines approach \influenceMinimization as a discrete combinatorial optimization problem. See Appendix~E for more details on the baselines.

\paragraph{\method.}
For each training graph, we generated 1,000 random seed sets, using 800 for training and 200 for validation.
For each test graph, we generated 50 random seed sets and reported average performance, where we used 10,000 Monte Carlo simulations as the ``ground-truth'' influence, following the settings by~\citet{kempe2003maximizing}.
We consistently used a graph convolutional network (GCN) with six layers and a final fully connected layer.
For \adv, the probabilistic decisions $\Tilde{r}$ were optimized (see Alg.~\ref{algo:adv}) in $n_{ep}=100$ epochs for each removal with $\alpha = 0.1$ and $\beta = 1$.

\paragraph{Seed set generations.}
For each seed set,
the size was sampled uniformly between 10 and $\lfloor 0.01|V| \rfloor$ (inclusive),
and then the nodes were sampled uniformly.

\paragraph{Probabilistic decisions.}
The probabilistic decisions $\Tilde{r}$ were normalized by a sigmoid function $\sigma_\text{sigmoid}(x)=\frac{1}{1+e^{-x}} \in [0, 1]$. We used initial probabilistic decisions $\tilde{r}(e)=\sigma_\text{sigmoid}(x)=1-\frac{b}{|E|},\forall e\in E$, which makes the loss $\calL_{\text{budget}}$ regarding the budget constraint to be $0$ (see Sec.~\ref{subsec:method:ours_optim}).
See Appendix~F for more details of the experimental settings, e.g., hardware information.

\subsection{\textbf{Q1. Performance}}\label{sec:exp_perf}
We shall show that \method showed good performance in influence minimizing (IMIN; Problem~\ref{problem:rumor_blocking}) in terms of effectiveness (the reduced ratio of influence) and efficiency (the running time).
Formally, the reduced ratio $R_r$ is defined as
\begin{equation}\label{eq:reduced_ratio}
 R_r(\calE; G, p, S) \coloneqq \textstyle{\frac{\sigma(S; G, p)-\sigma(S; G_{\setminus{\mathcal{E}}},p_{\setminus{\mathcal{E}}})}{\sigma(S; G,p)-\Abs{S}}.}
\end{equation}
A higher $R_r$ implies higher effectiveness in \influenceMinimization.

\begin{figure}[t!]
    \centering
        \includegraphics[width=\linewidth]{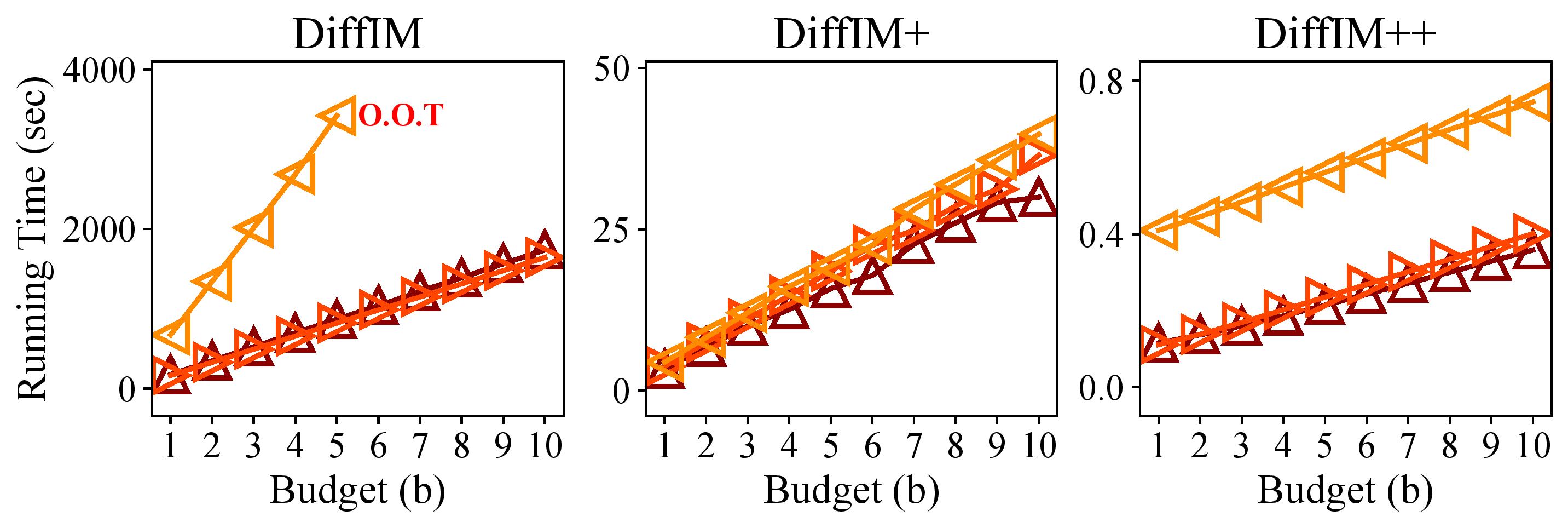}
        \includegraphics[width=0.4\linewidth]{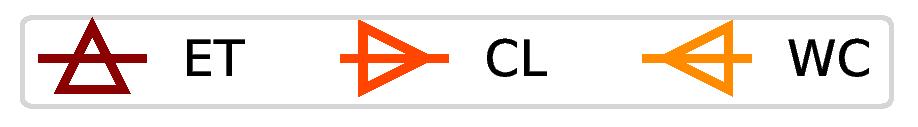}
    \caption{The running time of each \method version when budget $b$ increases from $1$ to $10$. 
    The running time of each version grew linearly with $b$, showing good scalability.}
    \label{fig:budget_6layer}
\end{figure}

\begin{figure}[t!]
    \centering
        \includegraphics[width=1\linewidth]{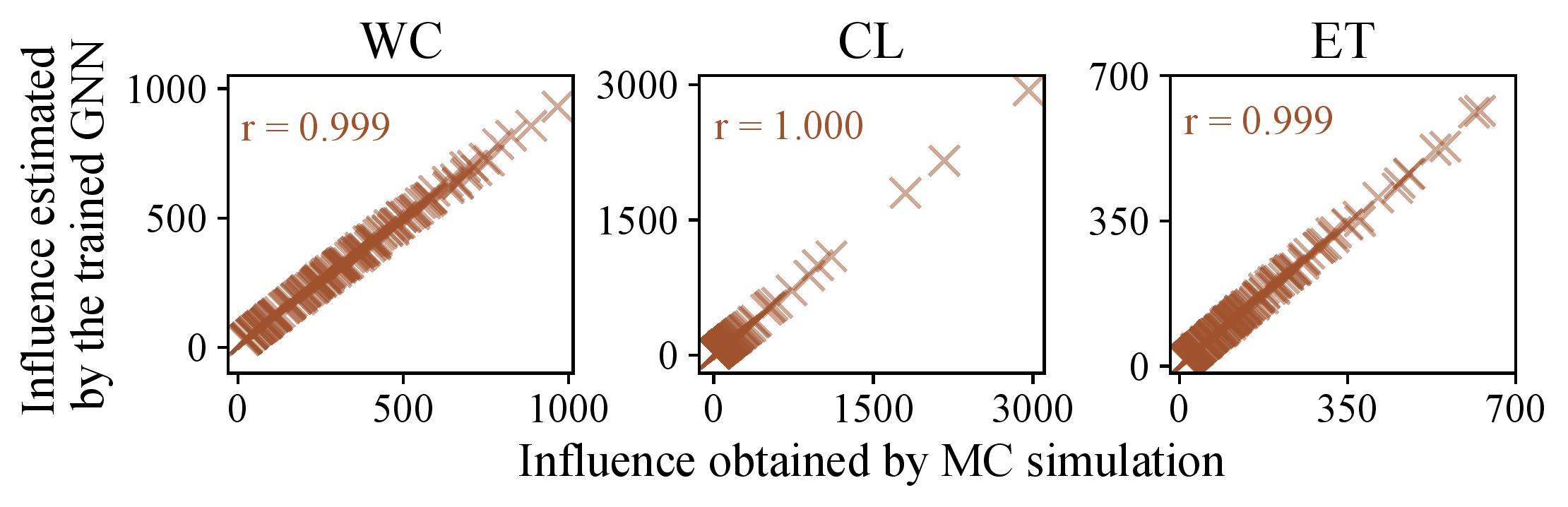}
    \caption{
    The Pearson correlation coefficients
    between the ground-truth influence of the validation sets and that estimated by MC simulation of a trained GNN.
    In all the cases, trained GNNs estimated influences near-perfectly.}
    \label{fig:Rvalue_6layer}
\end{figure}

In Fig.~\ref{fig:perf_6layer}, we reported the average reduced ratio of influence and running time across all the test seed sets for each dataset and method, using budgets $b \in \Set{5, 10}$.
Methods taking more than one hour on a single seed set were considered ``out of time'' and excluded.
Several baselines (\textsc{random}, \textsc{OdC}, BC, PR, \BPM, and \KED) do not depend on the seed set, and can be executed once for all seed sets.
Hence, we represented the performance of each of those methods as a horizontal line.
All the versions of \method were Pareto-optimal, i.e., no baseline was faster and more effective than any version, with at least one version outperforming all the baselines w.r.t. effectiveness in each case.
In most cases, all the versions outperformed all the baselines w.r.t. effectiveness.
Specifically, \method achieves 30 $\times$ to 15,160 $\times$ speed-up with similar effectiveness.
Notably, for \WL, \RIS-$\epsilon$ ran out of time for each $\epsilon \in \{0.2, 0.4, 0.6\}$.
As intended, the more proposed schemes we use, the higher speed we have (w.r.t. speed, \naive $<$ \adv $<$ \advp).
The effectiveness of different versions did not vary much, yet we observed an overall trend: \naive $>$ \adv $>$ \advp.
See Fig.~5 and Table~7 in Appendix~G.1 for the full results with standard deviations with different budgets $b \in \Set{3,5,7,10}$.

\subsection{\textbf{Q2. Scalability}}\label{sec:exp_sca}
In Fig.~\ref{fig:budget_6layer}, we reported the running time of all \method versions on each dataset with different budgets.
As the budget $b$ increased from 1 to 10, the running time increased almost linearly for all versions, which validated our analysis on the time complexity in Sec.~\ref{sec:anal}.
Our analysis in Sec.~\ref{sec:anal} also implies that \adv and \advp have better scalability w.r.t. the input graph size than \naive. 
Specifically, the time complexity of \naive is quadratic in $\Abs{E}$ while that of \adv and \advp is linear, which was also validated in Fig.~\ref{fig:budget_6layer},
where the running time gap between \naive and the other two versions increased on the larger dataset \WL.
In Table 9 in Appendix~G.2, for each method and each dataset, we provided the minimum budget $b$ for the method to run out of time (i.e., take more than one hour on a single seed set).

\subsection{Q3. Influence estimation quality}\label{sec:exp_train}
We shall show that our surrogate GNNs were trained well for influence estimation.
In Fig.~\ref{fig:Rvalue_6layer}, for the validation seed sets (recall that for each training graph, we generated 1,000 random seed sets, using 800 training and 200 for validation; see Sec.~\ref{subsec:exp:setting}), we reported the 
Pearson correlation coefficients (Pearson's $r$)
between the influence obtained by MC simulation (seen as the ground truth) and that estimated by the trained GNNs.
The estimation by the GNNs was highly correlated with the ground truth.
Specifically, the trained GNNs achieved a Pearson's $r$ of $0.999$ or higher on each dataset.
See Appendix~G.4 for more details, e.g., how the estimation errors decreased along training.

We reported the training times for GNNs in Table~\ref{tab:GNN_time_training} and compared the influence estimation times of GNNs with MC simulations in Table~\ref{tab:GNN_time_running}. 
Overall, the estimation times of GNNs were at least 100$\times$ faster than those of the MC simulations.

\begin{table}[t!]
    \centering  
    \begin{tabular}{c|r|r|r}
        \hline
        datasets & \multicolumn{1}{c|}{\WL}   & \multicolumn{1}{c|}{\CL}   & \multicolumn{1}{c}{\EL} \bigstrut \\
        \hline
        time (in seconds) &  17,227 & 7,390 & 6,818 \bigstrut \\
        \hline
    \end{tabular}
    \caption{Average training time (in seconds) for GNNs.}
    \label{tab:GNN_time_training}
\end{table}

\begin{table}[t!]
    \centering  
        \begin{tabular}{c|c|c|c}
        \hline
        estimation method   & \WL   & \CL   & \EL   \bigstrut \\
        \hline
        GNN & 0.0082 & 0.0056 & 0.0056 \bigstrut[b]  \bigstrut \\
        MC simulation & 3.7757 & 0.8276  & 0.7817 \bigstrut \\
        \hline
    \end{tabular}
    \caption{Average time (in seconds) to estimate influence for each seed set using GNNs and MC simulations. MC simulations were repeated 10,000 times for estimation.}
    \label{tab:GNN_time_running}
\end{table}

\begin{table*}[t!]
    \setlength{\tabcolsep}{1mm}
    \centering
    
    \begin{tabular}{l||c|c|c||c|c|c||c|c|c}
        \hline
         method & \multicolumn{3}{c||}{\naive} &\multicolumn{3}{c||}{\adv} & \multicolumn{3}{c}{\advp} \bigstrut \\ 
        \hline
        dataset & \WL & \CL & \EL & \WL & \CL & \EL & \WL & \CL & \EL  \bigstrut \\
        \hline
        transductive
        &0.4311 &0.6547 &0.5613 
        &0.3914 &0.6614 &0.5332 
        &0.3876 &0.6583 &0.4718 \bigstrut[t] \\
        inductive
        &0.4256 &0.6394 &0.5429 
        &0.3910 &0.6534 &0.5045 
        &0.3692 &0.6230 &0.5050 
        \\
        (difference)&(-1.3\%)&(-2.3\%)&(-3.3\%)
        &(-0.1\%)&(-1.2\%)&(-5.4\%)
        &(-4.7\%)&(-5.4\%)&(+7.0\%) \bigstrut[b]\\
        \hline
        \makecell[l]{strongest baseline (transductive)}
        &0.3160 &0.5591 &0.5284 
        &0.3160 &0.5591 &0.5284 
        &0.3160 &0.5591 &0.5284 
        \bigstrut \\
    \hline
    
    \end{tabular}
    \caption{The effectiveness (the reduced ratio of influence) of each \method version {when budget $b=5$}, and the comparison between the performance in transductive and inductive settings, with each percentage being the difference ratio (negative means the effectiveness in inductive settings was lower).
    In all the cases, the effectiveness of each \method version was only slightly lower (or even higher) in the inductive setting. Importantly, \method outperformed all the baselines in most cases in both transductive and inductive settings.
    \label{tab:induct_6layer}
    }
\end{table*}

\begin{table}[t]
    \centering
    \begin{tabular}{l|c|c|c}
        \hline
        loss & \WL & \CL & \EL  \bigstrut \\
        \hline
        original     &\textbf{0.3914} &\textbf{0.6614} &\textbf{0.5332} \bigstrut[t]\\
        $-\mathcal{L}_{\text{budget}}$     &0.0052 &0.0112 &0.0170 \\
        $-\mathcal{L}_{\text{certainty}}$  &0.3891 &0.6479 &0.5282 \bigstrut[b] \\
    \hline
    \end{tabular}
    \caption{The effectiveness (the reduced ratio of influence) of \adv on budget $b=5$ when removing different parts of the loss function (Eq.~\eqref{eq:loss_optimize}).
    Each part was helpful, and $\calL_{\text{budget}}$ was much more helpful than that of $\calL_{\text{certainty}}$.
    }
    \label{tab:ab_6layer}
\end{table}

\subsection{Q4. Inductivity}\label{sec:exp_induct}
In Table~\ref{tab:induct_6layer}, we compared the effectiveness (the reduced ratio of influence; see Eq.~\eqref{eq:reduced_ratio}) of all \method  versions in the transductive setting 
(trained and tested on the same dataset) 
and the inductive setting
(trained and tested on different datasets) {when budget $b=5$}.
Note that even for the transductive setting, the training and test graphs were different (spec., different timestamps), and the seed sets were different (see Sec.~\ref{subsec:exp:setting}).
In the inductive setting, on each dataset, we tested the two models that were trained on the other two datasets. 
The reported performance was the average of the results from these two models.
In most cases, the effectiveness of our methods was higher in the transductive setting than in the inductive setting.
From transductive to inductive settings, their effectiveness dropped by at most 5.4\%.
In most cases and settings, the effectiveness of our methods is higher than the most effective baseline (which was transductive and much slower than \adv and \advp), showing the good inductivity of our methods.

\subsection{Q5. Ablation studies}\label{subsec:exp:ablation}
We evaluated the importance of $\calL_{\text{budget}}$ and $\calL_{\text{certainty}}$ in the loss function of \adv (see Eq.~\eqref{eq:loss_optimize}).
In Table~\ref{tab:ab_6layer}, we compared the effectiveness (the reduced ratio of influence) of \adv on budget $b=5$ when using the whole loss function and when removing $\calL_{\text{budget}}$ or $\calL_{\text{certainty}}$.
We observed the effectiveness of \adv droped significantly without $\calL_{\text{budget}}$, showing the significance of $\calL_{\text{budget}}$.
{The effectiveness of $\calL_{\text{certainty}}$ was marginal, but it was necessary for theoretical guarantees (see Lem.~2 in Appendix~A.3).}

\subsection{Additional experiments}
We also conducted experiments on \textbf{two more influence diffusion models}, specifically, the linear threshold (LT) model~\citep{kempe2003maximizing} and the general Markov chain susceptible-infected-recovered (G-SIR) model~\citep{yi2022edge}.
On these models, \method still showed empirical superiority, being significantly faster than the most effective baseline, while achieving a similar reduced ratio of influence.
See Appendix~B for more details.

In addition, we conducted experiments on large-scale datasets.
Due to the absence of realistic activation probabilities, for these datasets, we used the weighted cascade model~\citep{kempe2003maximizing}, a special case of the IC model where the activation probability of each edge from node $u$ to $v$ is $1$ divided by the in-degree of $v$.
On these datasets, many strong baselines ran out of time or memory, and \method consistently outperformed those that completed within the given limits.
See Appendix~G.5 for details.

Furthermore, in Appendix~G, we presented additional experimental results on (1) the detailed trade-off between time and reduction ratio, (2) scalability w.r.t. budgets, and (3) comparisons with variants of \method missing certain components. Overall, we demonstrate the superiority of \method over baseline methods in terms of trade-offs and scalability, and the importance of each component in \method.

\section{Conclusions}
\label{sec:conclusion}
In this work, we studied influence minimization (\influenceMinimization) with edge removal under the independent cascade (IC) model, with more models discussed in Appendix~B.
We proposed \method (Sec.~\ref{sec:method}), which incorporates two key schemes:
\textit{surrogate modeling} for efficient influence estimation (Sec.~\ref{subsec:method:ours_greedy}) 
and \textit{continuous relaxation} of edge removal (Sec.~\ref{subsec:method:ours_optim}).
Additionally, we proposed \textit{gradient-driven edge selection} for instant edge selection without test-time gradient descent iterations (Sec.~\ref{subsec:method:ours_instant}).
Our extensive experiments demonstrated that all three schemes improved the speed of \method with little (or even no) \influenceMinimization performance degradation, in addition to its superior speed and effectiveness over baselines (Sec.~\ref{sec:exp_perf}). 
We also showed its scalability (Sec.~\ref{sec:exp_sca}) and 
ability to perform well when trained and tested on different graphs (Sec.~\ref{sec:exp_induct}). 

Our future work will extend our approach to other influence-related graph problems.
For example, our method can be adapted to the influence maximization problem~\citep{kempe2003maximizing}, whose objective is to identify the most influential seed set, by (1) introducing a global seed node linked to all existing nodes with 100\% activation probabilities and (2) selectively removing some of these new edges to maximize the influence of the seed node.


\section*{Acknowledgements}
This work was supported by the National Research Foundation of Korea (NRF) grant funded by the Korea government (MSIT) (No. RS-2024-00406985, 50\%). This work was supported by Institute of Information \& Communications Technology Planning \& Evaluation (IITP) grant funded by the Korea government (MSIT) (No. 2022-0-00871 / RS-2022-II220871, Development of AI Autonomy and Knowledge Enhancement for AI Agent Collaboration, 40\%) (No. RS-2019-II190075, Artificial Intelligence Graduate School Program (KAIST), 10\%).

\bibliography{aaai25}

\input{appendix/appendix}

\end{document}

%% file: appendix/appendix.tex
\newcommand{\smallsection}[1]{\noindent\underline{\smash{\textbf{#1:}}}}

\clearpage
\appendix
\vspace{1mm}
\begin{center}
    \Large{\bf \method: APPENDIX} \\
    \vspace{-1mm}
\end{center}

\section{Proofs}\label{sec:pfs_A}
\subsection{Proof of Theorem~\ref{thm:np_hard}}\label{app:proof:np_hard}
\begin{proof}
    We consider an NP-hard problem called minimum $k$-union~\citep{vinterbo2002note}, and
    we shall show that given each instance of minimum $k$-union, we can construct an instance of influence minimization, such that if the instance of influence minimization is our problem is solved, then the original instance of minimum $k$-union is solved.
    Given a collection of $m$ set $S_1, S_2, \ldots, S_m$ and a positive integer $k$, the minimum $k$-union problem aims to find $k$ sets ($S_{i_1}, S_{i_2}, \ldots, S_{i_k}$) such that $\vert \bigcup_{j} S_{i_j} \vert$ is minimized.
    Let 
    \[X = \{x_1, x_2, \ldots, x_n\} = \bigcup_{j \in [m]}S_j.\]
    Given any instance of minimum $k$-union, we construct the following instance of influence minimization:
    we have a single seed node $S = \{v_{init}\}$,
    $m$ nodes $v^{(S)}_{j}$ for $j \in [m]$, and
    $n$ nodes $v^{(X)}_{i}$ for $i \in [n]$.
    We have $m$ edges from $v_{init}$ to each $v^{(S)}_{j}$, 
    and edge from $v^{(S)}_{j}$ to $v^{(X)}_{i}$ if and only if $x_i \in S_j$, for each $(i, j)$ pair.
    Also, each edge $e$ has an activation probability $p(e) = 1$.
    Now, it is easy to see that the original instance of minimum $k$-union is equivalent to finding $m - k$ edges between $v_{init}$ and $v^{(S)}_{j}$ such that removing those $m - k$ edges will minimize the expected number of activated nodes in the whole process (i.e., the objective in influence minimization).
    Hence, it suffices to show that you cannot do better by removing other edges.
    Indeed, if you remove an edge from $v^{(S)}_{j}$ to $v^{(X)}_{i}$, then replacing it by removing the edge from $v_{init}$ to $v^{(S)}_{j}$ will give at least the same, or better, minimization performance.
\end{proof}

\subsection{Proof of Lemma~\ref{lem:gnn_diff}}\label{app:proof:gnn_diff}
\begin{proof}
We have
\[
\frac{\partial \GNN_{\theta}(v; G, \Tilde{p}_{\Tilde{r}}, S)}{\partial \Tilde{r}} = \frac{\partial \GNN_{\theta}(v; G, \Tilde{p}_{\Tilde{r}}, S)}{\partial \Tilde{p}_{\Tilde{r}}} \frac{\partial \Tilde{p}_{\Tilde{r}}}{\partial \Tilde{r}},
\]
    where    
    \[
    \frac{\partial \Tilde{p}_{\Tilde{r}}(v, u)}{\partial \Tilde{r}(v', u')} = p(v, u) \mathbb{1}(v = v' \land u = u').
    \]    
    As mentioned in Sec.~\ref{subsec:method:ours_greedy}, $p$ is used as edge weights in the input graph of GNN, and thus $\frac{\partial \GNN_{\theta}(v; G, \Tilde{p}_{\Tilde{r}}, S)}{\partial \Tilde{p}_{\Tilde{r}}}$ is also well-defined.
    Hence, $\frac{\partial \GNN_{\theta}(v; G, \Tilde{p}_{\Tilde{r}}, S)}{\partial \Tilde{r}}$ is well-defined.
\end{proof}

\subsection{Proof of the meaningfulness of the loss Function}\label{app:proof:relax_good}
\begin{lemma}\label{lem:relax_good}
    Given a test instance $\calT = (G = (V, E), p, S, b)$, let 
    \[
    \sigma^* \coloneqq \min_{\calE \subseteq E, \Abs{\calE} = b} \sigma(S; G_{\setminus \calE}, p_{\setminus \calE}).
    \]
    Assume that 
    (1) $\GNN_\theta$ is well trained, i.e.,
    \[
    \Abs{\GNN_\theta(v; G, p, S) - \pi(v; G, p, S)} \leq \epsilon, \forall v
    \]
    for some $\epsilon > 0$ and
    (2) $\alpha$ and $\beta$ are sufficiently large,
    then each $r^* \in \arg \min_{\Tilde{r}} \calL_O(\Tilde{r})$ satisfies that
    (1) $r^*$ is discrete, 
    i.e., 
    \[
    r^*(e) \in \Set{0, 1}, \forall e \in E,
    \]
    (2) $r^*$ satisfies the budget constraint,
    i.e.,
    \[
    \sum_{e \in E} r^*(e) = \Abs{E} - b,
    \]
    and
    (3) $r^*$ is a good solution, i.e., 
    \[
    \sigma(S; G_{\setminus \calE^*}, p_{\setminus \calE^*}) - \sigma^* \leq 2\Abs{V}\epsilon,
    \]
    where 
    \[
    \calE^* = \calE^*(r^*) = \Set{e \in E \colon r^*(e) = 0}.
    \]
\end{lemma}

\begin{proof}
    Since $\alpha$ and $\beta$ are sufficiently large and there exists
    $\tilde{r}: E \to [0, 1]$ such that 
    (1) $\calL_{\text{budget}}$ is minimized, i.e.,
    \[
    \calL_{\text{budget}}(\tilde{r}) = 0,
    \]
    and
    (2) $\calL_{\text{certainty}}$ is minimized, i.e.,
    \[
    \calL_{\text{certainty}}(\tilde{r}) = 0,
    \]
    $r^*$ must satisfy that $\calL_{\text{budget}}(r^*) = \calL_{\text{certainty}}(r^*) = 0$, which is equivalent to $r^*(e) \in \Set{0, 1}, \forall e \in E$ and $\sum_{e \in E} r^*(e) = \Abs{E} - b$.
    In fact, all the discrete and budget-satisfying probabilistic decisions $\Tilde{r}$ satisfy that     
    $\calL_{\text{budget}}(\Tilde{r}) = \calL_{\text{certainty}}(\Tilde{r}) = 0$,
    and we have 
    \[
    r^* \in \arg \min_{\Tilde{r}: E \to \Set{0, 1}, \sum_{e \in E} \tilde{r}(e) = \Abs{E} - b} \calL_{\text{obj}},
    \]
    which is equivalent to 
    \[
    r^* \in \arg \min_{\Tilde{r}: E \to \Set{0, 1}, \sum_{e \in E} \tilde{r}(e) = \Abs{E} - b} \sum_{v \in V} \GNN_\theta(v; G, \Tilde{p}_{\Tilde{r}}, S).
    \]    
    Define $\calE_{r} \coloneqq \Set{e \in E \colon r(e)}$ for each $r: E \to \Set{0, 1}$, and 
    define
    \[    
    \Tilde{r}_{\calE}(e) = \mathbb{1}(e \notin \calE), \forall e \in E.
    \]
    Let
    \[
    \calE_{min} \in \arg \min_{\calE \subseteq E, \Abs{\calE} = b} \sigma(S; G_{\setminus \calE}, p_{\setminus \calE}),
    \]    
    for each $v \in V$, we have
    \begin{align*}
     \pi(v; G_{\setminus \calE_{r^*}}, p_{\setminus \calE_{r^*}}, S) 
     &\leq \GNN_\theta(v; G, \Tilde{p}_{r^*}, S) + \epsilon \\
     &\leq \GNN_\theta(v; G, \Tilde{p}_{\Tilde{r}_{\calE_{min}}}, S) + \epsilon \\
    & \leq \pi(v; G_{\setminus \calE_{min}}, p_{\setminus \calE_{min}}, S) + 2\epsilon.
    \end{align*}
    Taking the summation over $v \in V$ completes the proof.
\end{proof}

\begin{remark}\label{rem:regularization}
    {In practice, however, using too large $\alpha$ and $\beta$ would make $\calL_{\text{budget}}$ and $\calL_{\text{certainty}}$ dominant and thus impair the optimization performance w.r.t. the main objective $\calL_{\text{obj}}$.
    See Sec.~\ref{subsec:exp:ablation} for the ablation studies on $\alpha$ and $\beta$.} 
\end{remark}

\section{Extension to other spread models}\label{sec:diff_model}

We performed additional experiments under different realistic influence spread models: the linear threshold (LT) model~\citep{kempe2003maximizing} and the general Markov chain susceptible-infected-recovered (G-SIR) model~\citep{yi2022edge}.

\smallsection{Model definition}
Below, we describe the definitions of the LT model and the G-SIR model.

\begin{definition}[Linear threshold (LT) model]\label{def:LT}
    Given 
    (1) a graph $G = (V, E)$ and
    (2) a seed set $S \subseteq V$,
    $\operatorname{LT}(G, S)$ is a stochastic process as follows:
    \begin{itemize}[leftmargin=*,topsep=0pt]
        \item \textbf{Initialization:}         
        At time step $t = 0$, each seed node $v_S \in S$ is activated, and each non-seed node remains inactive. For each $v \in V$, activation threshold $p_v$, drawn from continuous uniform distribution $U(0, 1)$, is assigned to $E$.
        \item \textbf{Diffusion steps:}
        At each step $t \geq 1$,
        each inactive node $v$ until the previous step is newly activated when the ratio of its active in-neighbors exceeds the activation threshold $p_v$.
        Each activated node remains active for the whole process, and the process terminates when no node is activated in the previous step.
    \end{itemize}
\end{definition}

\begin{definition}[General Markov chain susceptible-infected-recovered (G-SIR) model]\label{def:GSIR}
    Given 
    (1) a graph $G = (V, E)$,
    (2) activation probabilities $p: E \rightarrow [0,1]$,
    (3) a recovery probability $r$,
    and
    (4) a seed set $S \subseteq V$,
    $\operatorname{G-SIR}(G, S)$ is a stochastic process as follows:
    \begin{itemize}[leftmargin=*,topsep=0pt]
        \item \textbf{Initialization:}         
        At time step $t = 0$, each seed node $v_S \in S$ is activated, and each non-seed node remains inactive.
        \item \textbf{Diffusion steps:}
        At each step $t \geq 1$,
        each active node $v$ at the previous step becomes inactive again with the recovery probability $r$ and activates each of its inactive out-neighbor $u$ with activation probability $p(v, u)$. 
        Each recovered node remains inactive for the whole process.
        As long as each activated node stays activated, other non-recovered nodes can be activated during the process. The process terminates when no node is activated in the previous step.
    \end{itemize}
\end{definition}

In the G-SIR model, each edge has its individual (and possibly different) propagation probability, while in the original SIR model, the propagation probability is the same for all edges. 
In our experiments, we set the recovery probability $r = 0.5$.

\smallsection{Extensions}
Each considered method (see Sec.~\ref{subsec:exp:setting}) was straightforwardly extended for the two additional influence spread models.
For example, for extending \method, we modified the influence-estimation component (i.e., the surrogate-model GNN; see Alg.~\ref{algo:gnn}) to use the corresponding influence spread model.

\smallsection{Experimental results}
As in Sec.~\ref{sec:exp_perf}, we measured the average reduced ratio of influence and the average running time across all the test seed sets. 
We used the same process described in Sec.~\ref{subsec:exp:setting} to generate the training and test data for the LT and G-SIR models. 
As shown in Fig.~\ref{fig:app_perf_others}, the results followed a similar trend presented in Sec. \ref{sec:exp_perf}. 
That is, \advp was significantly faster than the most effective baseline, while achieving a similar reduced ratio.
For \WL, \RIS-$\epsilon$ ran out of time for every $\epsilon \in \{0.2, 0.4, 0.6\}$, leaving no reasonably comparable baselines.

\begin{figure}[t!]
    \centering
    \begin{subfigure}{0.95\linewidth}
        \includegraphics[width=\linewidth]{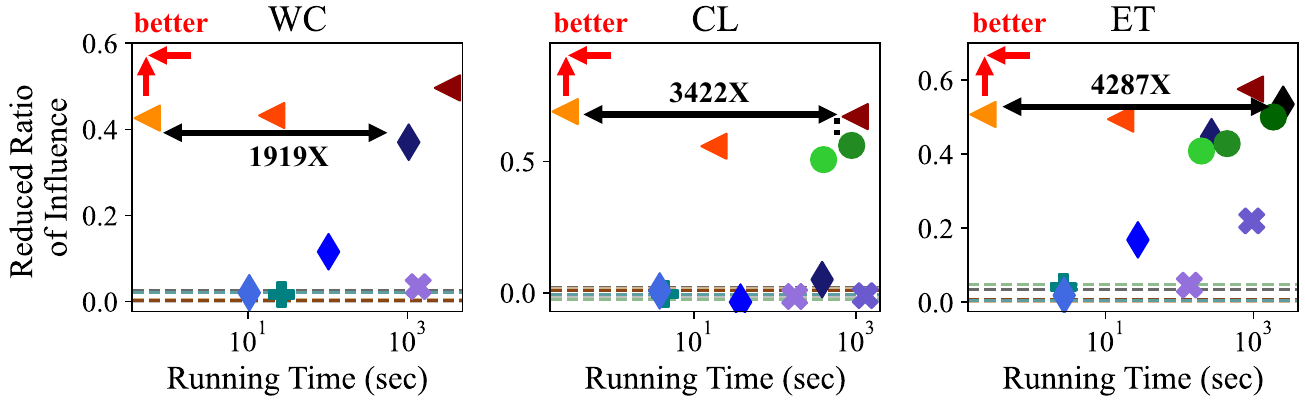}
        \caption{LT, $b=5$}
    \end{subfigure}
    \begin{subfigure}{0.95\linewidth}
            \includegraphics[width=\linewidth]{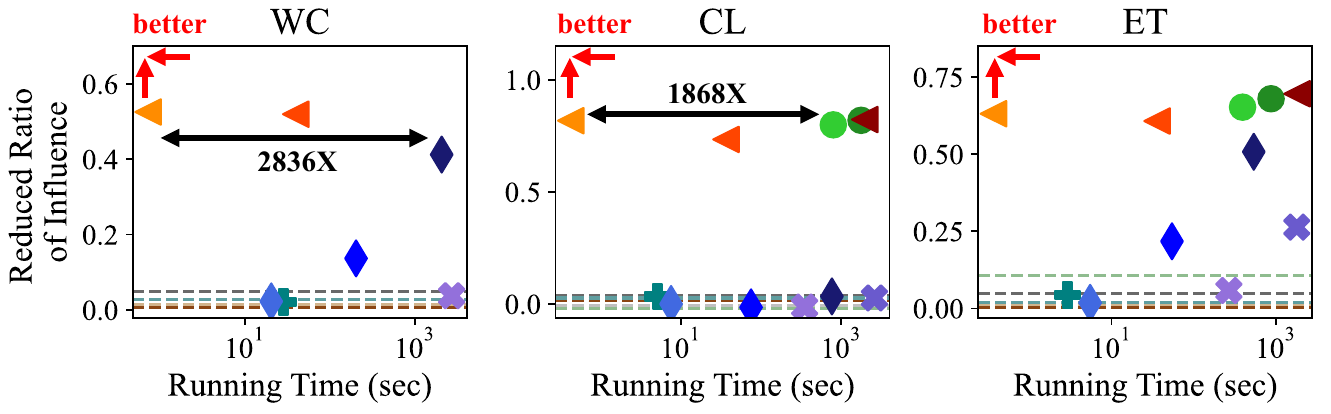}
        \caption{LT, $b=10$}
    \end{subfigure}
    \begin{subfigure}{0.95\linewidth}
        \includegraphics[width=\linewidth]{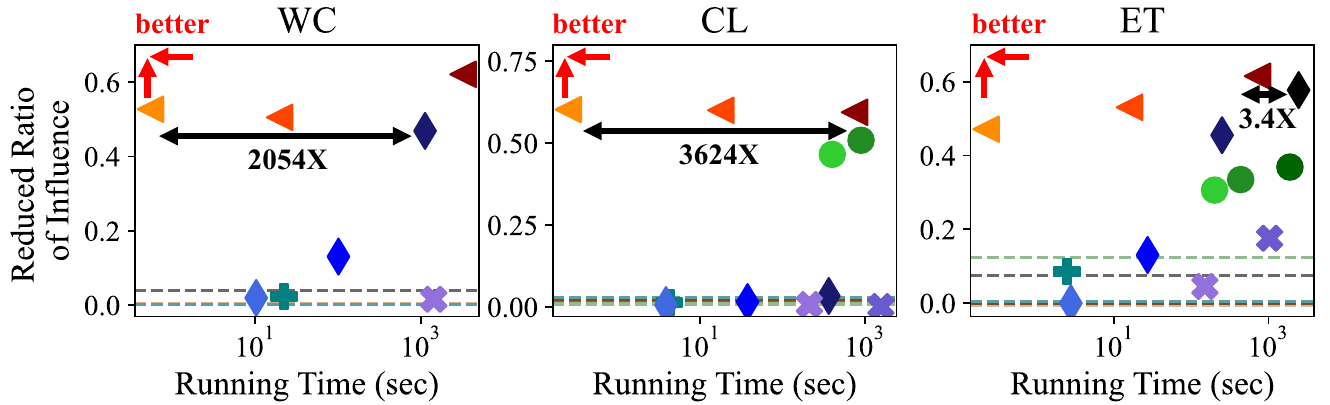}
        \caption{G-SIR, $b=5$}
    \end{subfigure}
    \begin{subfigure}{0.95\linewidth}
        \includegraphics[width=\linewidth]{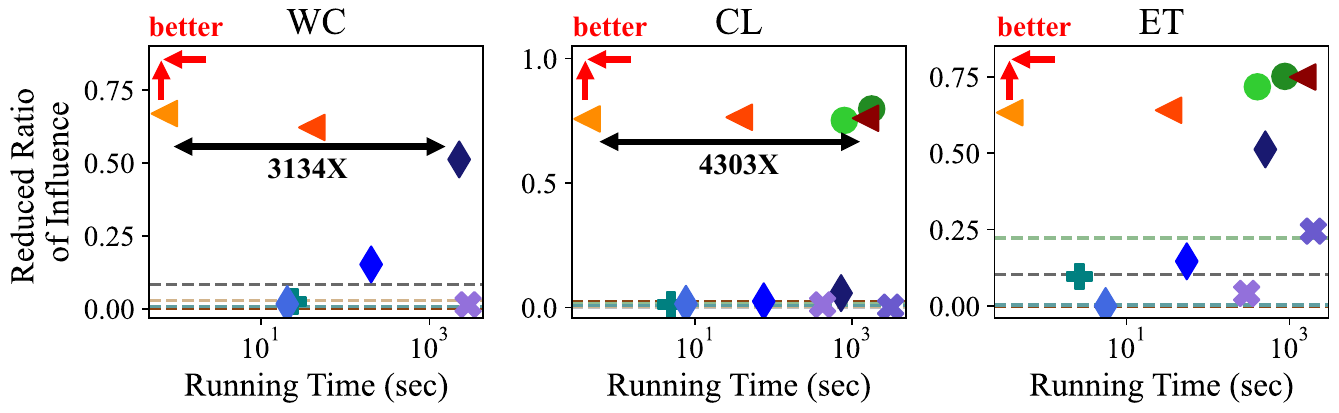}
        \caption{G-SIR, $b=10$}
    \end{subfigure}
    \begin{subfigure}{0.95\linewidth}
        \includegraphics[width=\linewidth]{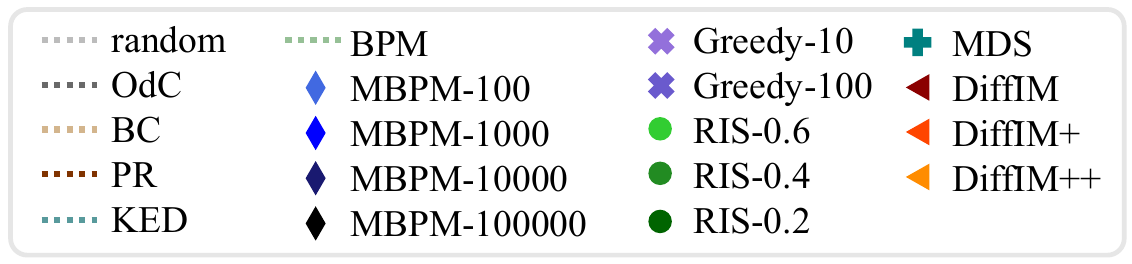}
    \end{subfigure}
    
    \caption{The effectiveness (the reduced ratio of influence) and running time of each method, with budget $b \in \Set{5, 10}$ on LT and G-SIR models. We compared the running time of the best baseline to one of our methods with the most similar reduced ratio. Except for (d) on \CL and \EL, all \method versions were Pareto-optimal, i.e., no baseline is faster \textit{and} more effective than any version.
    }
    \label{fig:app_perf_others}
\end{figure}


\input{appendix/algos/greedy}
\input{appendix/algos/gnn_train}

\section{Pseudo-codes of algorithms}\label{sec:alg}
We provide the pseudo-codes of the naive incremental greedy algorithm (Alg.~\ref{algo:greedy}; see Sec.~\ref{subsec:method:naive_greedy}) and GNN training (Alg.~\ref{algo:gnn}; see Sec.~\ref{subsec:method:ours_greedy}).

\section{Additional related work}\label{app:rel_wk_detail}
Here, we provide more details of the related work on influence minimization.

\smallsection{Node removal}
The following works considered influence minimization with node removal.
\citet{wang2013negative} employed an incremental greedy algorithm under the IC model.
\citet{chang2016guarantee} focused on the outbreak phenomenon under the IC model and addressed the problem of maximizing the probability that the influence remains below a certain threshold, employing a greedy algorithm based on Monte Carlo simulation.
\citet{zheng2018least} employed a greedy algorithm to remove bridge end nodes to prevent propagation to other communities.
\citet{zhu2021misinformation} removed node groups from given candidate node groups in an acyclic graph under a propagation model that extends the IC model by incorporating ECE (echo chamber effect).
\citet{xie2023minimizing} accelerated the influence estimation in the greedy algorithm using a dominator tree.
\citet{ni2023misinformation} considered protecting specific nodes in multi-social networks by forcing them not to be activated.

\smallsection{Edge removal}
The following works considered influence minimization with edge removal.
\citet{Kimura2009blocking} minimized the average (or maximum) influence when each node is a seed set under the IC model using a greedy algorithm and accelerated the process with the bond percolation method.
\citet{tong2012gelling} minimized the eigenvalue of the graph under the susceptible-infectious-susceptible (SIS) model.
\citet{khalil2014scalable} minimized the influence when one of the nodes in the given source set became a seed node under the LT model and employed a greedy algorithm that quickly calculates marginal loss using a live-edge tree.
\citet{yao2015minimizing} applied a greedy algorithm to the influence minimization problem identical to the one considered by us (i.e., Problem~\ref{problem:rumor_blocking}).
\citet{yan2019rumor} addressed a problem identical to ours but with the condition of acyclic graphs. They analyzed the problem from the perspective of marginal decrement and proposed a heuristic algorithm based on propagation ability.
\citet{yi2022edge} accelerated the greedy algorithm in the G-SIR model by using reverse influence sampling.
\citet{zareie2022rumour} minimized the spreading ability of the graph when the seed set is not given.
\citet{wang2020efficient} proposed a robust sampling-based greedy algorithm to protect a given target node in the LT model.
\citet{jiang2022rumordecay} protected given target nodes in the SIR model, by identifying critical edges through sampling paths from the seed nodes to the target nodes

\smallsection{Active defense}
The following works proposed active defense to counter negative propagation.
\citet{budak2011limiting} used a Monte Carlo simulation-based greedy algorithm in the IC model and compared it with some heuristics.
\citet{he2012influence} estimated the influence of each node by using a locally directed acyclic graph in the LT model.
\citet{fan2013least} selected bridge ends to protect other communities from propagation using a greedy algorithm in the opportunistic one-activate-one model and deterministic one-activate-many model.
\citet{luo2014time} employed a Monte Carlo simulation-based greedy algorithm in the continuous-time multiple campaign diffusion model.
\citet{zhu2016minimum} efficiently calculated single-hop spread to estimate influence in the IC model.
\citet{tong2019beyond} developed a hybrid sampling process in the IC model that attaches high weights to the users vulnerable to misinformation.
\citet{hosni2019darim} used a greedy algorithm that simultaneously performs node blocking and active defense.


\input{appendix/algos/mds}
\input{appendix/algos/mbpm}

\section{Additional details of baseline methods}\label{sec:app_base}
\begin{enumerate}[leftmargin=*,topsep=0pt,itemsep=0pt]
    \item \textbf{\MDS}~\cite{yan2019rumor} estimates the incremental change in the influence when each edge is removed, and greedily chooses $b$ edges w.r.t. the incremental changes. In Alg.~\ref{algo:MDS}, we provide the pseudo-code of MDS. 
    The edges are chosen according to the influenced probability and rumor-spread ability of their endpoints, with incremental updates. 
    \item \textbf{BPM}~\cite{Kimura2009blocking} uses the bond percolation method (BPM) to estimate the importance of edges, and removes the top-$b$ edges w.r.t. importance scores. BPM does not consider specific seed nodes.
    \item \textbf{Modified \BPM (\MBPM)} is a modified version of \BPM that considers specific seed nodes.
    Specifically, we count the number of nodes reachable from the seed set $S$, in a sampled graph according to the activation probabilities $p$. We provide its pseudo-code in Alg.~\ref{algo:MBPM}.
    \item \textbf{\RIS}~\cite{yi2022edge} randomly samples a node and an activated graph each round, checks for each edge whether the node is reachable from the set of seeds without traversing the edge, and removes bottom-$b$ edges w.r.t. the success rate. Let $\tilde{\sigma}(S; G, p)$ be the number of infected nodes (including recovered nodes for the G-SIR model~\cite{yi2022edge}) except for the seed set. 
    According to Proposition 8.1 in~\cite{yi2022edge}, a $(1\pm \epsilon)$-approximation of $\tilde{\sigma}(S; G, p)$ is obtained with $O(\epsilon^{-2}n\log n)$ rounds of sampling an activated graph and a node. Therefore, as the error term $\epsilon$ decreases, a larger number of rounds is required.
    In our experiments, we conduct $0.1\cdot\epsilon^{-2}n\log n$ rounds for $\epsilon=0.2$, $0.4$, and $0.6$, respectively.
\end{enumerate}

\section{Additional details of experimental settings}\label{sec:app_setting}

\smallsection{\method}
We implemented \method in Python with the PyG library~\citep{fey2019fast} using the Adam optimizer~\cite{kingma2014adam}.
The initial node features were one-dimensional binary: $1$ if the node is a seed node, and $0$ otherwise.
For GNN training (Alg.~\ref{algo:gnn}), the learning rate and its decaying rate were optimized by Optuna~\cite{optuna_2019} with 2000 epochs.

\smallsection{Hardware}
All the versions of \method were run on a machine with 
2.10GHz Intel\textsuperscript{\textregistered} Xeon\textsuperscript{\textregistered} Silver 4210R processors and RTX2080Ti GPUs.
A single GPU was used for each experiment.
The baseline methods did not use GPUs, and they were run on a machine with more powerful CPUs, 3.70GHz Intel\textsuperscript{\textregistered} Core\textsuperscript{\texttrademark} i9-10900KF processors.

\section{Additional experimental results}\label{sec:app_exp}

\subsection{Additional results for Q1. Performances}\label{sec:app_perf}
We reported the effectiveness (the reduced ratio of influence) and the running time, along with the standard deviations, for each method, for each budget $b \in \Set{3,5,7,10}$, in Fig.~\ref{fig:app_perf} and Table~\ref{tab:app_perf}.
The standard deviations were computed over different seed sets, which could be high since the seed sets could be highly different. 


\subsection{Budget scalability}\label{app:scal} 
We conducted experiments with all the considered methods with budget $b$ increasing from $1$ to $10$.
In Table~\ref{tab:budget_min}, for each method and each dataset, we reported the minimum value of $b$ with which the method ran out of time (i.e., took more than one hour on a single seed set), where ``.''  indicates that the method did not run out of time even with $b = 10$.
The methods whose outputs do not depend on the seed set were not included.
Among the versions of \method, only \naive failed on the largest dataset \WL.
Among the baseline methods, MBPM had the lowest budget scalability, and \Greedy did not scale well on large graphs.

\subsection{Variants of \method}\label{app:together1by1} 
For \adv, one can optimize $\Tilde{r}$ and then choose the bottom-$b$ edges with the lowest $\Tilde{r}$ values together instead of removing edges one by one for each $n_{ep}$ epochs.
In Fig.~\ref{fig:together1by1}, we reported the effectiveness of the variant (selecting edges together after $bn_{ep}$ epochs) compared to original \adv (selecting edges one by one) with different budgets $b\in [10]$ for each dataset.
The effectiveness of the variant of \adv was nearly the same as the original \adv.

For \advp, one can choose the top-$b$ edges with the largest gradient together.
In Fig.~\ref{fig:together1by1_I}, we reported the effectiveness of the variant (selecting edges together instantly) compared to the original \advp (selecting edges one by one).
The effectiveness of the variant of \advp was noticeably lower than that of original \advp, except on the CL dataset.

\subsection{Estimation errors along training}\label{app:estim_error}
As shown in Fig.~\ref{fig:val_error_6layer}, the estimation errors on the validation set decreased as the training proceeded.
Overall, the trained GCNs achieved good estimation quality, with errors at most 4.5\% of the ground-truth influence.

\subsection{Performance on large-scale datasets}\label{sec:larger_dataset}

We additionally conducted experiments on three large-scale datasets: cit-HepTh~\citep{leskovec2005graphs, gehrke2003overview}, email-EuAll~\citep{leskovec2007graph}, and twitter~\citep{leskovec2012learning}. All these datasets are available at SNAP~\citep{snapnets}. We provided their basic statistics in Table~\ref{tab:larger_data}.
Due to the absence of realistic activation probabilities, for these datasets, we used the weighted cascade model~\citep{kempe2003maximizing}, a special case of the IC model where the activation probability of each edge from node $u$ to $v$ is $1$ divided by the in-degree of $v$.

The results with a budget of $b=5$ and $b=10$ are presented in Figure~\ref{fig:app_perf_larger}. The proposed method consistently and significantly outperformed the baseline methods that completed within the time and memory limits. Here, \BPM, \naive(naive), \RIS, and \MBPM-100000 ran out of time.


\input{appendix/tabs/full_res}


\input{appendix/figures/full_res}

\FloatBarrier


\begin{table}[h!]
  \centering
  \scalebox{0.8}{%
    \begin{tabular}{l||c|c|c}
        \hline
        methods & \WL & \CL & \EL  \bigstrut \\
        \hline
        \hline        
        \MDS & .  &. &. \bigstrut \\
        \hline
        \MBPM-10000 & .  &. &.  \bigstrut[t] \\
        \MBPM-100000 & 3 & 5 & 7 \bigstrut[b] \\
        \hline
        \Greedy-10 & . & . & . \bigstrut[t] \\
        \Greedy-100 & 2  &. &. \bigstrut[b] \\
        \hline
        \RIS-0.6 & 5 & . & . \bigstrut[t]\\
        \RIS-0.4 & 2 & . & . \\
        \RIS-0.2 & 1 & 6 & 10 \bigstrut[b] \\
        \hline
        \naive & 6  &. &. \bigstrut[t] \\
        \adv & .  &. &. \\
        \advp & .  &. &. \bigstrut[b]\\
    \hline
    \end{tabular}
    }
    \caption{The minimum budget $b \leq 10$ with which each method runs out of time. Each cell with ``.'' implies that the method did not run out of time even with $b = 10$.}
    \label{tab:budget_min}
\end{table}

\begin{figure}[h!]
    \centering
    \includegraphics[width=\linewidth]{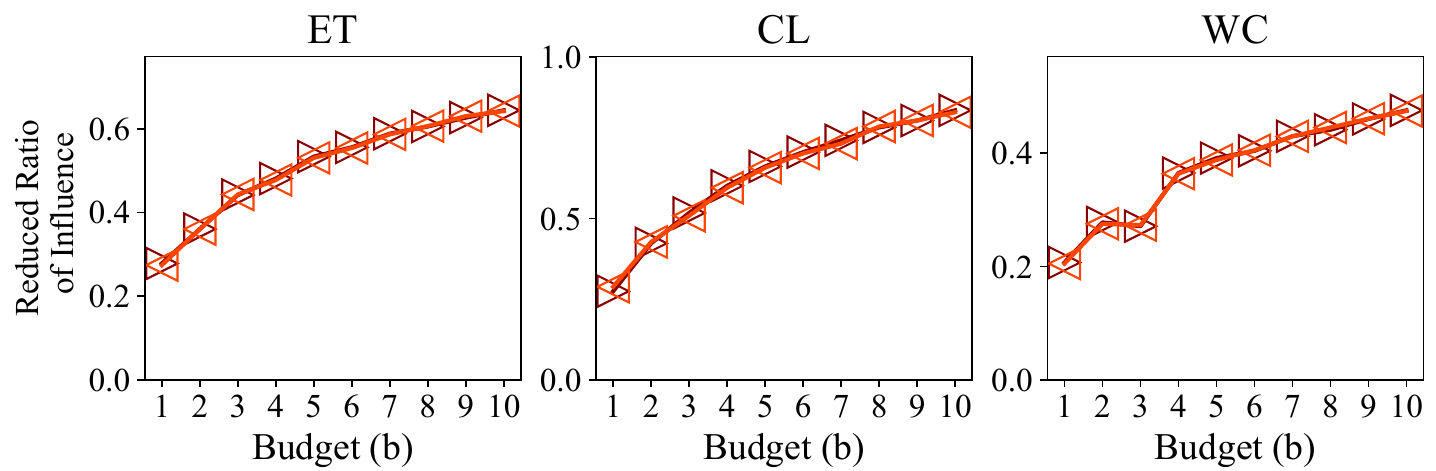}
    \includegraphics[width=0.7\linewidth]{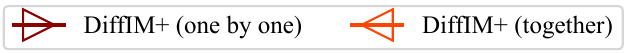}
    \caption{The effectiveness of the variant (selecting edges together after $bn_{ep}$ epochs) of \adv compared to original \adv (selecting edges one by one) with respect to budgets $b \in [10]$ for each dataset. The effectiveness of the variant was nearly the same as the original \adv.}
    \label{fig:together1by1}
\end{figure}

\begin{figure}[h!]
    \centering
    \includegraphics[width=\linewidth]{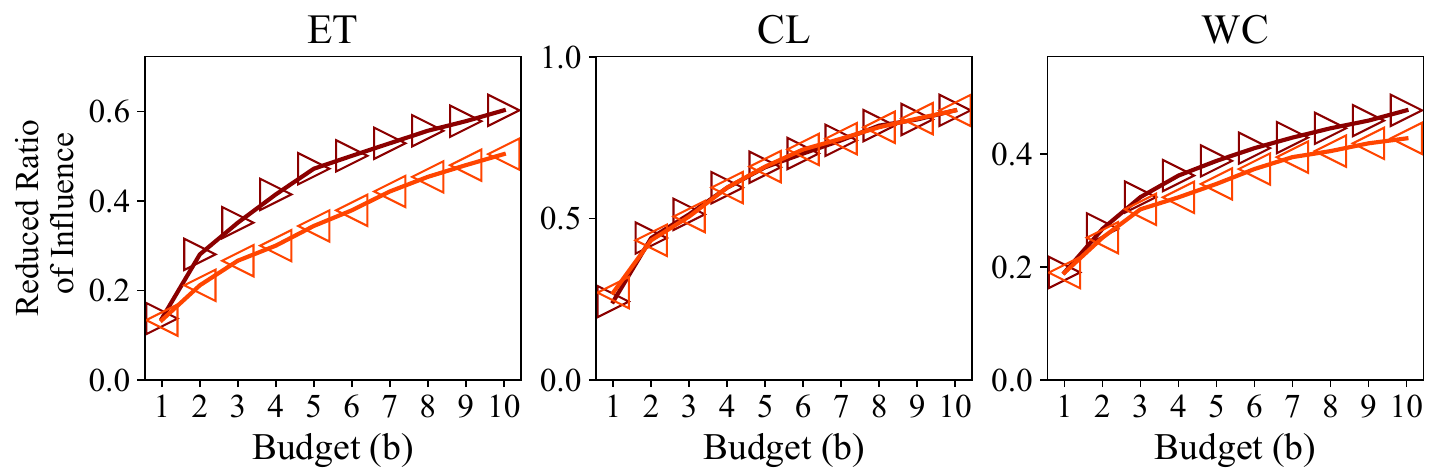}
    \includegraphics[width=0.7\linewidth]{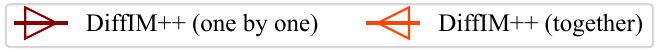}
    \caption{The effectiveness of the variant (selecting edges together) of \advp compared to original \advp (selecting edges one by one) with respect to budgets $b \in [10]$ for each dataset. The effectiveness of the variant is lower than that of the original \adv, except on the CL dataset.}
    \label{fig:together1by1_I}
\end{figure}

\begin{figure}[h!]
   \centering
   \includegraphics[width=0.75\linewidth]{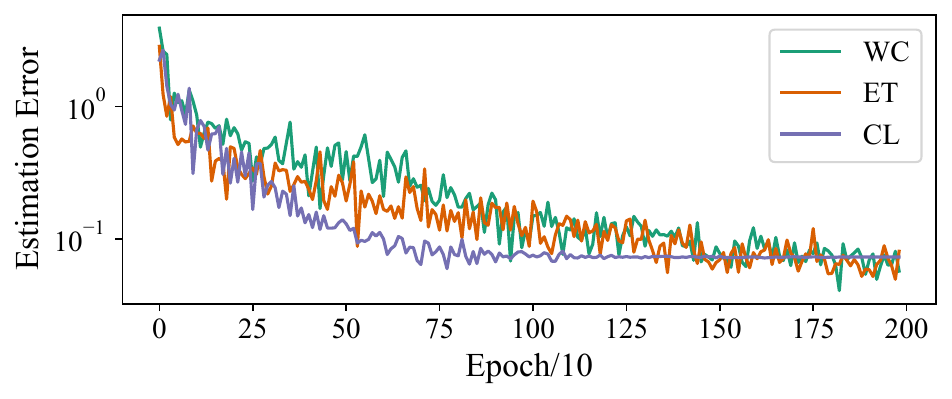}
   \caption{The average validation estimation errors in estimating influence decreased as GCN training proceeds.}
   \label{fig:val_error_6layer}
\end{figure}%

\begin{figure}[h!]
    \centering
    \includegraphics[width=\linewidth]{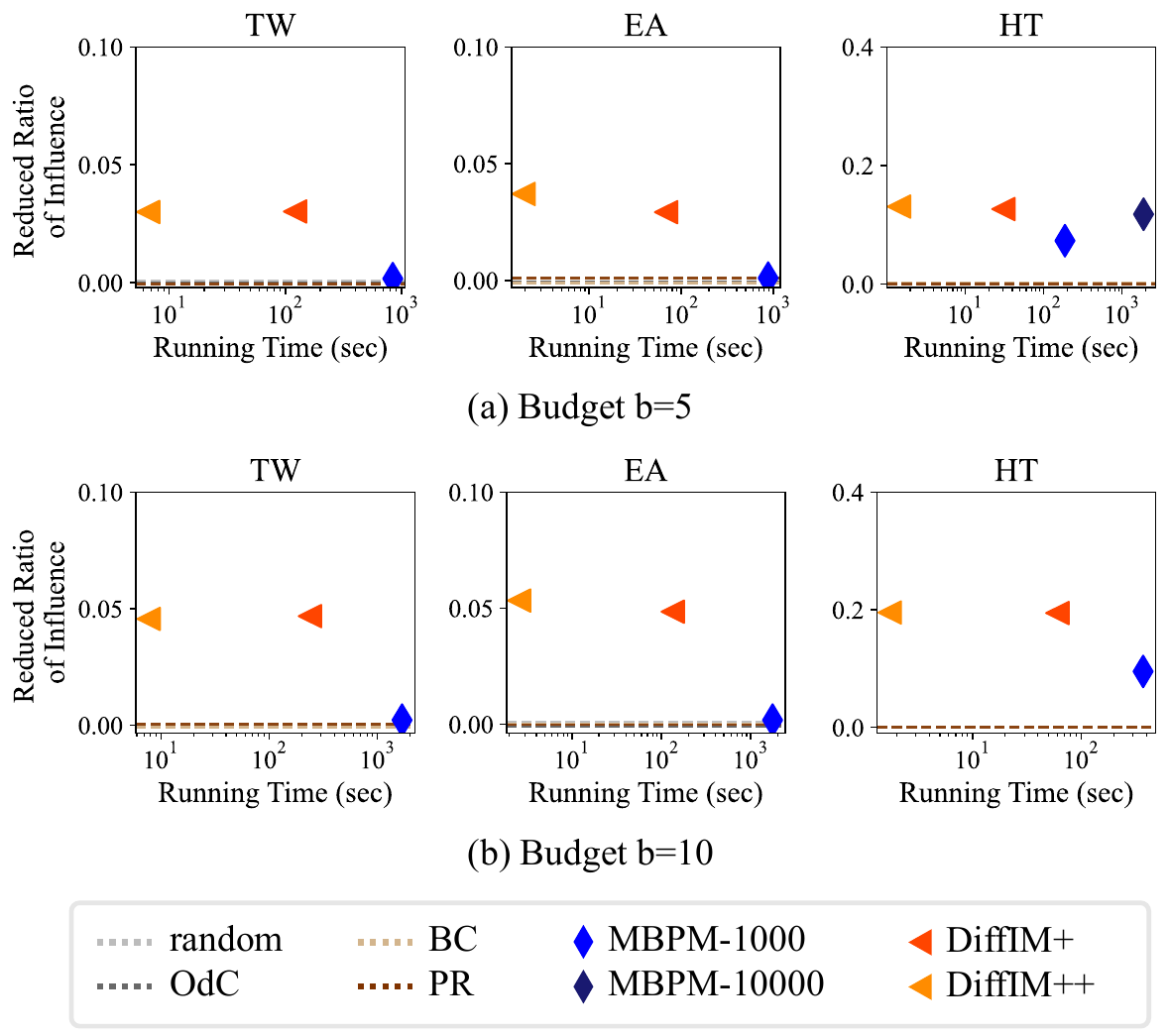}
    \caption{The effectiveness (the reduced ratio of influence) and running time of each method, with budget $b=5$ (top) and $b=10$ (bottom) on the large-scale datasets
    }
    \label{fig:app_perf_larger}
\end{figure}

\begin{table}[h!]
    \centering
    \begin{tabular}{l|l|r|r}
        \hline
        dataset &abbr. &\textbf{$|V|$}&\textbf{$|E|$} \bigstrut\\
        \hline
        cit-HepTh   & HT & $27,770$ & $352,807$\\
        email-EuAll & EA & $265,214$ & $420,045$\\ 
        twitter     & TW & $81,306$ & $1,768,149$\\
        \hline
    \end{tabular}
    \caption{The basic statistics of the large-scale datasets.}
    \label{tab:larger_data}
\end{table}

%% file: appendix/algos/greedy.tex
\begin{algorithm}[t]
    \caption{An incremental greedy algorithm}\label{algo:greedy}
    \SetKwInput{KwInput}{Input}
    \SetKwInput{KwOutput}{Output}
    \KwInput{(1) $G =  (V, E)$: an input graph \\
    \quad\quad\quad (2) $p: E \to [0, 1]$: activation probabilities \\
    \quad\quad\quad (3) $S \subseteq V$: a seed set \\ 
    \quad\quad\quad (4) $b$: an edge-removal budget \\}
    \KwOutput{$\mathcal{E} \subseteq E$: a set of edges chosen to be removed}
    $\mathcal{E} \leftarrow \emptyset$ \label{algo:greedy:init} \hfill $\triangleright$ Initialization \\     
    \For{$i = 1, 2, \ldots, b$}{
        $e=\arg\min_{e\in E\setminus \calE}\sigma(S; G_{\setminus \calE}, p_{\setminus \calE})$ \label{algo:greedy:find}\\
        $\mathcal{E}\leftarrow \mathcal{E} \cup \{e\}$\label{algo:greedy:update}\hfill $\triangleright$ Incremental update \\        
    }
 \Return $\mathcal{E}$
 \label{algo:greedy:end}
\end{algorithm}

%% file: appendix/algos/gnn_train.tex
\begin{algorithm}[t]
    \caption{{GNN training for influence estimation (executed once, applicable to all future unseen cases)}} \label{algo:gnn}
    \SetKwInput{KwInput}{Input}
    \SetKwInput{KwOutput}{Output}
    \KwInput{(1) $G=(V, E)$: an input graph \\
    \quad\quad\quad (2) $p: E \to [0,1]$: activation probabilities \\
    \quad\quad\quad (3) $\calS = \Set{S_1, S_2, \ldots, S_k}$: a set of seed sets \\
    \quad\quad\quad (4) $\GNN_{\theta}$: a GNN to be trained \\
    }
    \KwOutput{$\GNN_{\theta}$: a trained GNN}
    {Obtain} $\tilde{\pi}(v; G, S_i, p)$ by MC, for each $v \in V$ and $S_i \in \calS$ \label{algo:gnn:inf_MC}\\ 
    Update $\theta$ w.r.t. $\calL(\GNN_{\theta}; \tilde{\pi}, \calS, G, p)$ 
    \label{algo:gnn:update_param} \\
 \Return $\GNN_{\theta}(\cdot)$ \label{algo:gnn:return}  
\end{algorithm}

%% file: appendix/algos/mds.tex
\begin{algorithm}[htb]
    \small
    \caption{{MDS}}\label{algo:MDS}
    \SetKwInput{KwInput}{Input}
    \SetKwInput{KwOutput}{Output}    
    \KwInput{(1) $G =  (V, E)$: an input graph \\
    \quad\quad\quad (2) $p: E \to [0, 1]$: activation probabilities \\
    \quad\quad\quad (3) $S \subseteq V$: a seed set \\
    \quad\quad\quad (4) $b$: an edge-removal budget \\
    \quad\quad\quad (5) $h$: the number of  propagation hops}
    \KwOutput{$\mathcal{E}\subset E$: a set of edges chosen to be removed}
    $\textt{prob}(v) \leftarrow \pi(v; G,p,S)$ \hspace*{0pt}\hfill $\triangleright$\ {MC simulation} \\ 
    $A \leftarrow \text{the adjacent matrix of $G$}$\\
    $\textt{rsa} \leftarrow \sum_{0\le i\le h} A^i\mathbf{1}$\\
    \For{$i = 1,2,\ldots,b$}{
        $s_e\leftarrow (1-\textt{prob}(u))\cdot(\textt{rsa}(v)+\textt{rsa}(u)), \forall (v, u) \in E$\\
        $e \leftarrow \arg\max_{e\in E} s_e$\\
        $\textt{rsa} \leftarrow \textt{update\_rsa$(G,p,e, \textt{rsa})$}$\\
        $\textt{prob} \leftarrow \textt{update\_prob$(G,p,e, \textt{prob})$}$\\
        $E\leftarrow E \setminus \{e\}$;      $\mathcal{E}\leftarrow \mathcal{E}\cup\{e\}$\\
    }
    \Return $\mathcal{E}$
    
    \SetKwProg{Fn}{Function}{:}{\KwRet}
    \Fn{$\textt{update\_prob($G, p, e = (s,t), \texttt{prob}$)}$}{
        $\Delta\texttt{prob}(v)\leftarrow 0, \forall v\in V$\\
        $\Delta\texttt{prob}(t)\leftarrow \frac{p(s,t)\cdot \texttt{prob}(s)\cdot(1- \texttt{prob}(t))}{1-p(s,t)\cdot\texttt{prob}(s)}$\\
        $U \leftarrow \{t\}$; $V \leftarrow \{t\}$\\
        \While{$U\ne\emptyset$}{
            $U' \leftarrow \emptyset$\\
            \ForEach{$(u, v)\in E$ such that $u\in U, v \in V$}{
                $\Delta\texttt{prob}(v)\leftarrow \Delta\texttt{prob}(v)+\frac{p(u,v)\cdot\Delta\texttt{prob}(u)\cdot(1-\texttt{prob}(v))}{1-p(u,v)\cdot\texttt{prob}(u)}$\\
                \If{$v\notin V$}{
                    $U'\leftarrow U'\cup\{v\}$; $V\leftarrow V\cup\{v\}$
                }
            }
            $U \leftarrow U'$
        }
        $\texttt{prob}(v)\leftarrow\texttt{prob}(v)-\Delta\texttt{prob}(v), \forall v\in V$\\
        \Return $\texttt{prob}$
    }
    
    \SetKwProg{Fn}{Function}{:}{\KwRet}
    \Fn{$\textt{update\_rsa($G, p, e = (s, t), \texttt{rsa}$)}$}{
        $\Delta\texttt{rsa}(t)\leftarrow p(s,t)\cdot\texttt{rsa}(t)$\\
        $U \leftarrow \{s\}$; $V \leftarrow \{s\}$\\
        \While{$U\ne\emptyset$}{
            $U' \leftarrow \emptyset$\\
            \ForEach{$(v, u)\in E$ such that $v\in V, u\in U$}{
                $\Delta\texttt{rsa}(v)\leftarrow \Delta\texttt{rsa}(v)+p(v,u)\cdot\Delta\texttt{rsa}(u)$\\
                \If{$v\notin V$}{
                    $U'\leftarrow U'\cup\{v\}$; $V\leftarrow V\cup\{v\}$
                }
            }
            $U \leftarrow U'$
        }
        $\texttt{rsa}(v)\leftarrow\texttt{rsa}(v)-\Delta\texttt{rsa}(v), \forall v\in V$\\
        \Return $\texttt{rsa}$
    }
\end{algorithm}

%% file: appendix/algos/mbpm.tex
\begin{algorithm}[htb]
    \caption{{Modified BPM}}
    \label{algo:MBPM}
    \small
    \SetKwInput{KwInput}{Input}
    \SetKwInput{KwOutput}{Output}
    \KwInput{(1) $G =  (V, E)$: an input graph \\
    \quad\quad\quad (2) $p: E \to [0, 1]$: activation probabilities \\
    \quad\quad\quad (3) $S \subseteq V$: a seed set \\
    \quad\quad\quad (4) $b$: an edge-removal budget \\
    \quad\quad\quad (5) $d$: the number of samplings}
    \KwOutput{$\mathcal{E}\subset E$: a set of edges chosen to be removed}
    $\mathcal{E} \leftarrow \emptyset$\\     
    \For{$i = 1,2,\ldots,b$}{
        $s_e \leftarrow 0, n_e \leftarrow 0, \forall e\in E$\\        
        \For{$j = 1,2,\ldots,d$}{
            $l_e\sim \text{Bernoulli}(p(e)), \forall e\in E$\\ 
            $E'\leftarrow \{e\in E \colon l_e=1\}$\\
            $G' \leftarrow (V,E')$\\
            $n_S \leftarrow$ the number of nodes reachable from $S$ on $G'$ \label{algo:MBPM:reach}\\
            \ForEach{$e\in E \text{ such that } l_e=0$}{
                $s_e \leftarrow s_e+n_S$;            $n_e \leftarrow n_e+1$\\
            }
        }
        \ForEach{$e\in E$}{
            \If{$n_e\ge 1$}{$s_e\leftarrow s_e/n_e$}
            \Else {$s_e \leftarrow \infty$}
        }
        $e \leftarrow \arg\min_e s_e$\\
        $E \leftarrow E\setminus \{e\}$;        $\mathcal{E}\leftarrow\mathcal{E}\cup\{e\}$\\
    }
 \Return $\mathcal E$ 
\end{algorithm}

%% file: appendix/tabs/full_res.tex
\begin{table*}[h]
  \centering
  \scalebox{0.9}{%
    \begin{tabular}{l||c|c|c||c|c|c}
        \hline
        \multirow{2}{*}{method} & \multicolumn{3}{c||}{$b=3$} &\multicolumn{3}{c}{$b=5$} \bigstrut \\
        \cline{2-7}        
         & \WL & \CL & \EL  & \WL & \CL & \EL  \bigstrut\\
        \hline
        \hline
        \textsc{Random} & 0.0016 (0.0214) & -0.0133 (0.1105) & 0.0010 (0.0182)  &
        -0.0045 (0.0216) & 0.0010 (0.1188) & 0.0081 (0.0158) \bigstrut[t]\\
        OdC & 0.0022 (0.0225) & 0.0063 (0.1077) & 0.0123 (0.0207)  &
        0.0117 (0.0238) & 0.0352 (0.1534) & 0.0119 (0.0407)  \\
        BC & -0.0006 (0.0226) & 0.0006 (0.0860) & -0.0011 (0.0208)  &
        0.0019 (0.0220) & -0.0057 (0.0992) & -0.0028 (0.0206) \\
        PR &-0.0017 (0.0175) & -0.0036 (0.1204) & -0.0002 (0.0187)  &
        -0.0017 (0.0201) & -0.0010 (0.1077) & -0.0124 (0.0182) \bigstrut[b] \\
        \hline
        \BPM & \textbf{O.O.T} & 0.0159 (0.1221) & 0.0387 (0.0471)  &
        \textbf{O.O.T} & -0.0021 (0.1074) & 0.0477 (0.0527) \bigstrut[t] \\
        \KED &-0.0006 (0.0174) & 0.0109 (0.1162) & -0.0015 (0.0180)  &
        0.0029 (0.0178) & 0.0097 (0.1317) & -0.0029 (0.0556) \\
        \MDS &0.0098 (0.0254) & -0.0004 (0.1085) & 0.0298 (0.0366)  &
        0.0173 (0.0229) & 0.0202 (0.1341) & 0.0368 (0.0421) \bigstrut[b] \\
        \hline
        \MBPM-100 & 0.0105 (0.0622) & 0.0108 (0.1358) & 0.0093 (0.0610)  &
        0.0098 (0.0630) & -0.0076 (0.1215) & 0.0116 (0.0594) \bigstrut[t] \\
        \MBPM-1000 & 0.1342 (0.2195) & -0.0013 (0.1308) & 0.1396 (0.2096)  &
        0.1589 (0.2245) & -0.0052 (0.1187) & 0.1767 (0.2150) \\
        \MBPM-10000 & 0.2811 (0.2140) & 0.0136 (0.1350) & 0.3836 (0.2014)  &
        0.3160 (0.2170) & 0.0183 (0.1335) & 0.4273 (0.2127) \\
        \MBPM-100000 & \textbf{O.O.T} & 0.2046 (0.2338) & 0.4468 (0.1766)  &
        \textbf{O.O.T} & \textbf{O.O.T} & 0.5284 (0.1790) \bigstrut[b] \\
        \hline
        \Greedy-10 & 0.0059 (0.0481) & -0.0182 (0.0976) & 0.0244 (0.0995)  &
        0.0062 (0.0459) & -0.0052 (0.1040) & 0.0236 (0.1011) \bigstrut[t] \\
        \Greedy-100 & \textbf{O.O.T} & -0.0052 (0.1144) & 0.1560 (0.2363)  &
        \textbf{O.O.T} & 0.0052 (0.1374) & 0.1635 (0.2356) \bigstrut[b] \\
        \hline
        \RIS-0.6 & 0.2280 (0.1720) & 0.3813 (0.1941) & 0.2937 (0.1968) & \textbf{O.O.T} & 0.5045 (0.2064) &   0.3641 (0.2102) \bigstrut[t] \\
        \RIS-0.4 & \textbf{O.O.T} &  0.4428 (0.1960) & 0.3202 (0.2078) & \textbf{O.O.T} & 0.5227 (0.2136) & 0.4156 (0.2163)\\
        \RIS-0.2 & \textbf{O.O.T} & 0.4456 (0.1893) & 0.3688 (0.2100) & \textbf{O.O.T} &  0.5591 (0.2217) & 0.4513 (0.2212) \bigstrut[b]  \\
        \hline
        \naive & 0.3609 (0.1796) & 0.5240 (0.1876) & 0.4675 (0.1855)  &
        0.4311 (0.1726) & 0.6547 (0.1877) & 0.5613 (0.1828) \bigstrut[t] \\
        \adv  & 0.2713 (0.1978) & 0.5172 (0.1930) & 0.4415 (0.1871)  &
        0.3914 (0.1796) & 0.6614 (0.1957) & 0.5332 (0.1845)  \\
        \advp  & 0.3236 (0.1856) & 0.5125 (0.1956) & 0.3512 (0.2002)  &
        0.3876 (0.1759) & 0.6583 (0.1948) & 0.4718 (0.2066)  \\
        \hline        
        \multicolumn{7}{c}{}\\
        \hline        
        \multirow{2}{*}{method} & \multicolumn{3}{c||}{$b=7$} &\multicolumn{3}{c}{$b=10$} \bigstrut\\
        \cline{2-7}
         & \WL & \CL & \EL  & \WL & \CL & \EL  \bigstrut \\
         \hline
         \hline
        \textsc{Random}  &  -0.0011 (0.0187) & -0.0187 (0.1244) & 0.0014 (0.0210)  & 
        0.0002 (0.0201) & -0.0072 (0.1150) & 0.0005 (0.0182) \bigstrut[t] \\
        OdC  &  0.0275 (0.0330) & 0.0106 (0.1120) & 0.0315 (0.0439)  & 
        0.0327 (0.0338) & 0.0263 (0.1294) & 0.0448 (0.0494) \\
        BC  &  0.0027 (0.0194) & 0.0196 (0.1118) & 0.0090 (0.0557)  & 
        0.0020 (0.0234) & 0.0054 (0.1013) & 0.0158 (0.0534) \\
        PR  &  -0.0035 (0.0225) & 0.0187 (0.1071) & -0.0020 (0.0203) & 
        0.0006 (0.0239) & 0.0257 (0.1377) & 0.0002 (0.0218) \bigstrut[b] \\
        \hline
        \BPM  &  \textbf{O.O.T} & 0.0016 (0.0980) & 0.0882 (0.1145)  & 
        \textbf{O.O.T} & -0.0085 (0.1030) & 0.1047 (0.1432) \bigstrut[t] \\
        \KED  &  -0.0010 (0.0158) & 0.0017 (0.1316) & 0.0079 (0.0569)  & 
         -0.0012 (0.0197) & -0.0100 (0.1319) & 0.0113 (0.0579)\\
        \MDS  &  0.0160 (0.0244) & 0.0015 (0.1112) & 0.0372 (0.0442)  & 
        0.0163 (0.0258) & 0.0230 (0.1354) & 0.0396 (0.0424) \bigstrut[b] \\
        \hline
        \MBPM-100  &  0.0461 (0.1659) & 0.0091 (0.1202) & 0.0258 (0.1008)  & 
        0.0507 (0.1711) & -0.0080 (0.0778) & 0.0252 (0.1015)\bigstrut[t] \\
        \MBPM-1000  &  0.1587 (0.2318) & 0.0009 (0.1315) & 0.1778 (0.2219)  & 
        0.1641 (2315) & 0.0128 (0.1172) & 0.2014 (0.2220) \\
        \MBPM-10000  & 0.3365 (0.2215) & 0.0153 (0.1437) & 0.4673 (0.2089)  & 
        0.3545 (2243) & 0.0289 (0.1307) & 0.4991 (0.2072) \\
        \MBPM-100000  &  \textbf{O.O.T} & \textbf{O.O.T} & \textbf{O.O.T}  & 
        \textbf{O.O.T} & \textbf{O.O.T} & \textbf{O.O.T} \bigstrut[b] \\
        \hline
        \Greedy-10  &  0.0076 (0.0493) & -0.0152 (0.1241) & 0.0224 (0.1024)  & 
        0.0136 (0.0654) & -0.0300 (0.1453) & 0.0321 (0.1111) \bigstrut[t] \\
        \Greedy-100  &  \textbf{O.O.T} & 0.0030 (0.1299) & 0.1809 (0.2340)  & 
        \textbf{O.O.T} & -0.0074 (0.1414) & 0.1977 (0.2358) \bigstrut[b]\\
        \hline
        \RIS-0.6 & \textbf{O.O.T} & 0.5732 (0.2292) & 0.4172 (0.2134) & \textbf{O.O.T} & 0.7590 (0.2020) & 0.6390 (0.1443) \bigstrut[t]\\
        \RIS-0.4 & \textbf{O.O.T} & 0.6093 (0.2117) & 0.4688 (0.2131) & \textbf{O.O.T} & 0.8113 (0.1836) & 0.6517 (0.1534) \\
        \RIS-0.2 & \textbf{O.O.T} & \textbf{O.O.T} & 0.5116 (0.2195) & \textbf{O.O.T} & \textbf{O.O.T} & \textbf{O.O.T} \bigstrut[b] 
        \\
        \hline
        \naive  &  \textbf{O.O.T} & 0.7455 (0.1871) & 0.6171 (0.1757) & 
        \textbf{O.O.T} & 0.8124 (0.1870) & 0.6780 (0.1625) \bigstrut[t] \\
        \adv &  0.4301 (0.1754) & 0.7423 (0.2034) & 0.5881 (0.1794)  &
        0.4758 (0.1663) & 0.8352 (0.1957) & 0.6439 (0.1725)  \\
        \advp  &  0.4289 (0.1714) & 0.7417 (0.1914) & 0.5286 (0.2078)  & 
        0.4772 (0.1655) & 0.8346 (0.1919) & 0.6023 (0.1980) \\ 
    \hline
    \end{tabular}
    }
    \caption{The effectiveness (the reduced ratio of influence) of each method with the standard deviations, with budget $b \in \Set{3, 5, 7, 10}$.
  O.O.T denotes out-of-time, i.e., the method does not terminate within one hour on a single seed set in the corresponding setting.}
    \label{tab:app_perf}
\end{table*}

%% file: appendix/figures/full_res.tex
\begin{figure*}[t!]
    \centering
    \begin{subfigure}{0.65\linewidth}
        \includegraphics[width=\linewidth]{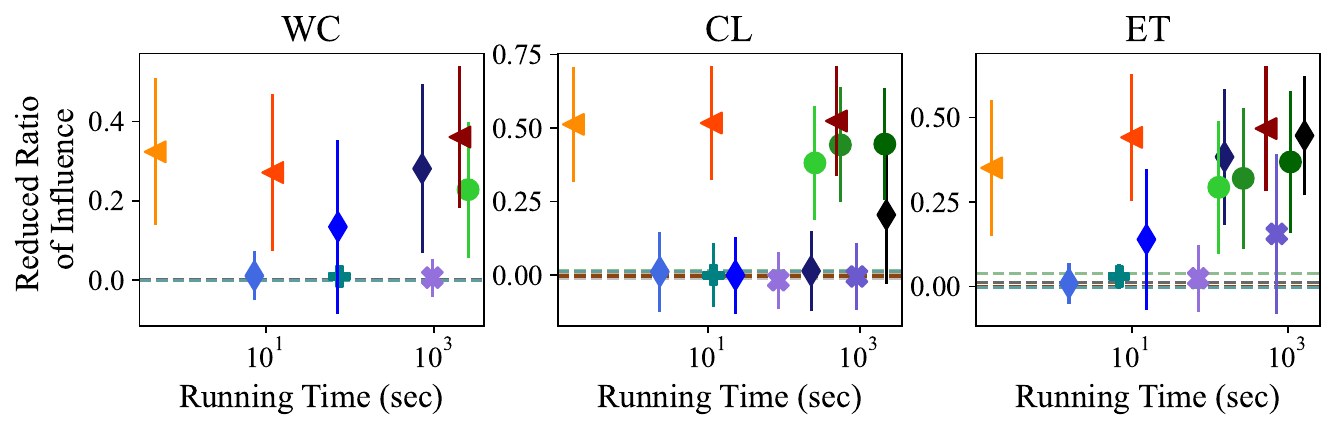}
        \caption{$b=3$}
    \end{subfigure} \\

    \begin{subfigure}{0.65\linewidth}
        \includegraphics[width=\linewidth]{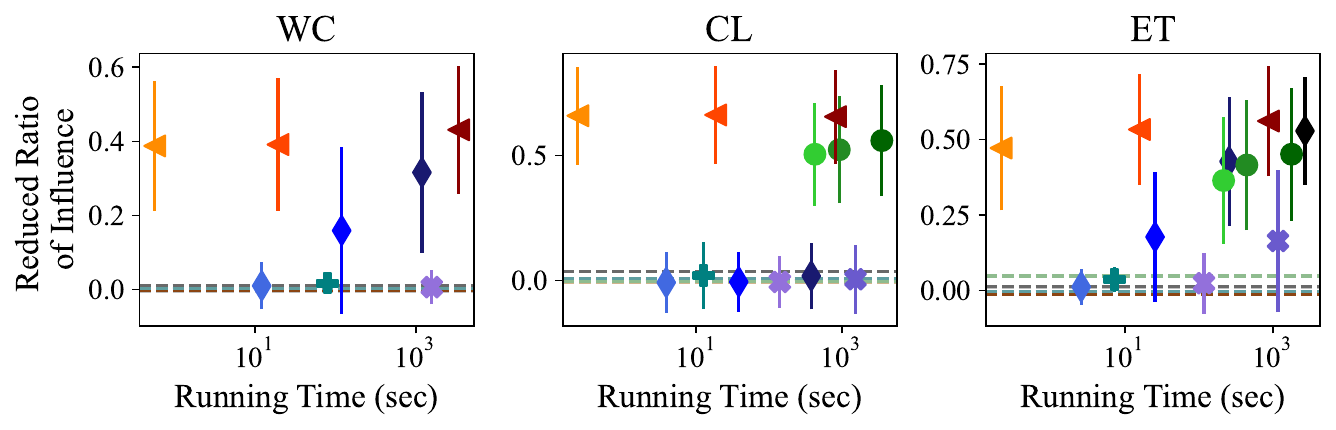}
        \caption{$b=5$}
    \end{subfigure} \\
    
    \begin{subfigure}{0.65\linewidth}
        \includegraphics[width=\linewidth]{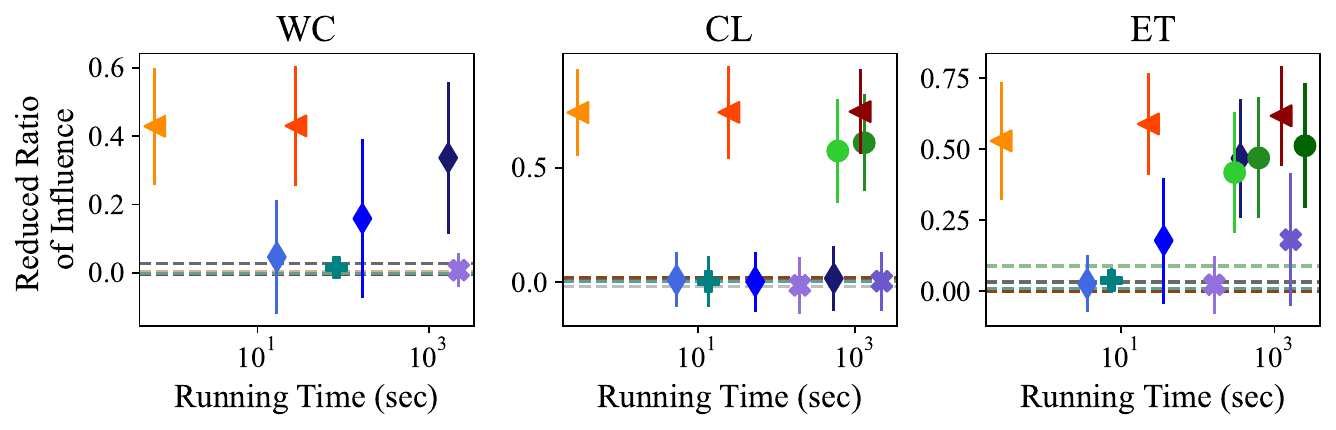}
        \caption{$b=7$}
    \end{subfigure} \\
    
    \begin{subfigure}{0.65\linewidth}
        \includegraphics[width=\linewidth]{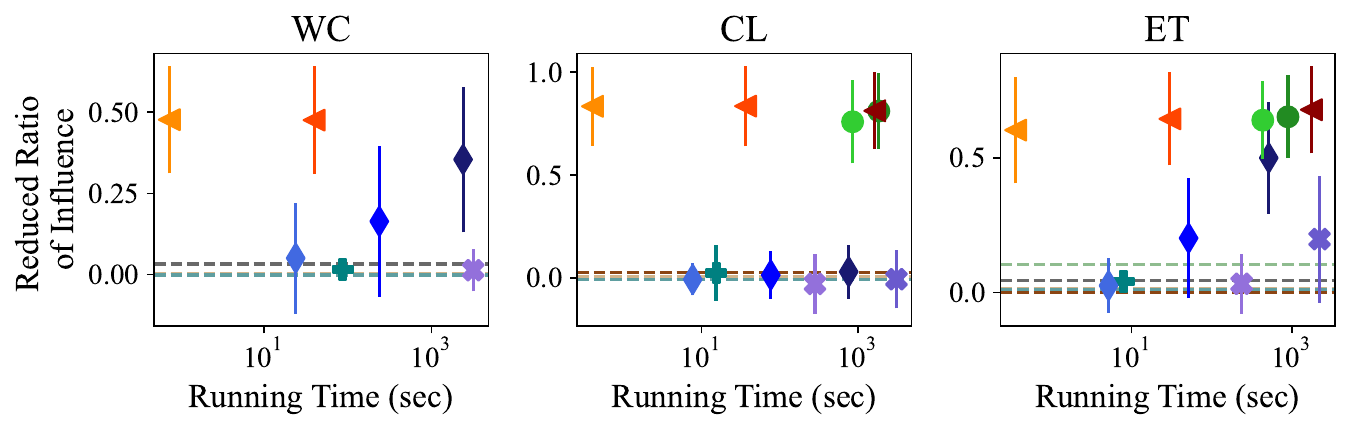}
        \caption{$b=10$}
    \end{subfigure}

    \begin{subfigure}{0.8\linewidth}
        \includegraphics[width=\linewidth]{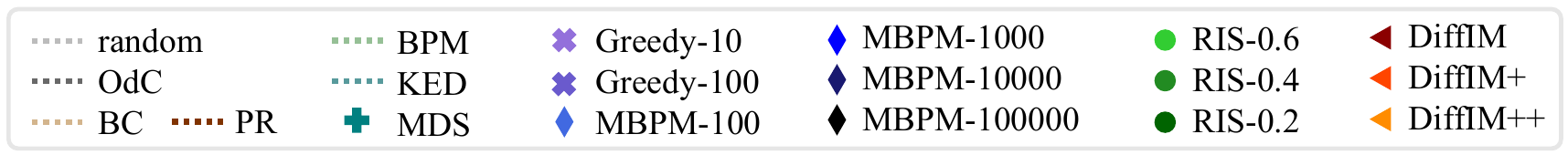}
    \end{subfigure}
    
    \caption{The effectiveness (the reduced ratio of influence) and running time of each method, with budget $b \in \Set{3, 5, 7, 10}$.
    The error bars represent one standard deviation.
    The superiority of \method is valid with all the values of $b$.}
    \label{fig:app_perf}
\end{figure*}

%% file: main_camera_ready.bbl
\begin{thebibliography}{52}
\providecommand{\natexlab}[1]{#1}

\bibitem[{Akiba et~al.(2019)Akiba, Sano, Yanase, Ohta, and Koyama}]{optuna_2019}
Akiba, T.; Sano, S.; Yanase, T.; Ohta, T.; and Koyama, M. 2019.
\newblock Optuna: A Next-generation Hyperparameter Optimization Framework.
\newblock In \emph{KDD}.

\bibitem[{Blakely, Lanchantin, and Qi(2021)}]{blakely2021time}
Blakely, D.; Lanchantin, J.; and Qi, Y. 2021.
\newblock Time and space complexity of graph convolutional networks.
\newblock Technical report, University of Virginia.

\bibitem[{Borgs et~al.(2014)Borgs, Brautbar, Chayes, and Lucier}]{borgs2014maximizing}
Borgs, C.; Brautbar, M.; Chayes, J.; and Lucier, B. 2014.
\newblock Maximizing social influence in nearly optimal time.
\newblock In \emph{SODA}.

\bibitem[{Budak, Agrawal, and El~Abbadi(2011)}]{budak2011limiting}
Budak, C.; Agrawal, D.; and El~Abbadi, A. 2011.
\newblock Limiting the spread of misinformation in social networks.
\newblock In \emph{WWW}.

\bibitem[{Chang, Yeh, and Chuang(2016)}]{chang2016guarantee}
Chang, C.-W.; Yeh, M.-Y.; and Chuang, K.-T. 2016.
\newblock On the guarantee of containment probability in influence minimization.
\newblock In \emph{ASONAM}.

\bibitem[{Chen, Wang, and Wang(2010)}]{chen2010scalable}
Chen, W.; Wang, C.; and Wang, Y. 2010.
\newblock Scalable influence maximization for prevalent viral marketing in large-scale social networks.
\newblock In \emph{KDD}.

\bibitem[{Chiang et~al.(2019)Chiang, Liu, Si, Li, Bengio, and Hsieh}]{chiang2019cluster}
Chiang, W.-L.; Liu, X.; Si, S.; Li, Y.; Bengio, S.; and Hsieh, C.-J. 2019.
\newblock Cluster-gcn: An efficient algorithm for training deep and large graph convolutional networks.
\newblock In \emph{KDD}.

\bibitem[{Dai et~al.(2022)Dai, Chen, Hu, and Ding}]{dai2022minimizing}
Dai, C.; Chen, L.; Hu, K.; and Ding, Y. 2022.
\newblock Minimizing the spread of negative influence in SNIR model by contact blocking.
\newblock \emph{Entropy}, 24(11): 1623.

\bibitem[{Fan et~al.(2013)Fan, Lu, Wu, Thuraisingham, Ma, and Bi}]{fan2013least}
Fan, L.; Lu, Z.; Wu, W.; Thuraisingham, B.; Ma, H.; and Bi, Y. 2013.
\newblock Least cost rumor blocking in social networks.
\newblock In \emph{ICDCS}.

\bibitem[{Fey and Lenssen(2019)}]{fey2019fast}
Fey, M.; and Lenssen, J.~E. 2019.
\newblock Fast graph representation learning with PyTorch Geometric.
\newblock \emph{arXiv:1903.02428}.

\bibitem[{Gehrke, Ginsparg, and Kleinberg(2003)}]{gehrke2003overview}
Gehrke, J.; Ginsparg, P.; and Kleinberg, J. 2003.
\newblock Overview of the 2003 KDD Cup.
\newblock \emph{Acm Sigkdd Explorations Newsletter}, 5(2): 149--151.

\bibitem[{He et~al.(2012)He, Song, Chen, and Jiang}]{he2012influence}
He, X.; Song, G.; Chen, W.; and Jiang, Q. 2012.
\newblock Influence blocking maximization in social networks under the competitive linear threshold model.
\newblock In \emph{SDM}.

\bibitem[{Hosni, Li, and Ahmad(2019)}]{hosni2019darim}
Hosni, A. I.~E.; Li, K.; and Ahmad, S. 2019.
\newblock DARIM: Dynamic approach for rumor influence minimization in online social networks.
\newblock In \emph{NeurIPS}.

\bibitem[{Jiang et~al.(2022)Jiang, Chen, Ma, and Philip}]{jiang2022rumordecay}
Jiang, Z.; Chen, X.; Ma, J.; and Philip, S.~Y. 2022.
\newblock RumorDecay: rumor dissemination interruption for target recipients in social networks.
\newblock \emph{IEEE Transactions on Systems, Man, and Cybernetics: Systems}, 52(10): 6383--6395.

\bibitem[{Jin et~al.(2013)Jin, Dougherty, Saraf, Cao, and Ramakrishnan}]{jin2013epidemiological}
Jin, F.; Dougherty, E.; Saraf, P.; Cao, Y.; and Ramakrishnan, N. 2013.
\newblock Epidemiological modeling of news and rumors on twitter.
\newblock In \emph{SNA-KDD}.

\bibitem[{Kempe, Kleinberg, and Tardos(2003)}]{kempe2003maximizing}
Kempe, D.; Kleinberg, J.; and Tardos, {\'E}. 2003.
\newblock Maximizing the spread of influence through a social network.
\newblock In \emph{KDD}.

\bibitem[{Kempe, Kleinberg, and Tardos(2005)}]{kempe2005influential}
Kempe, D.; Kleinberg, J.; and Tardos, {\'E}. 2005.
\newblock Influential nodes in a diffusion model for social networks.
\newblock In \emph{ICALP}.

\bibitem[{Khalil, Dilkina, and Song(2014)}]{khalil2014scalable}
Khalil, E.~B.; Dilkina, B.; and Song, L. 2014.
\newblock Scalable diffusion-aware optimization of network topology.
\newblock In \emph{KDD}.

\bibitem[{Kimura, Saito, and Motoda(2009)}]{Kimura2009blocking}
Kimura, M.; Saito, K.; and Motoda, H. 2009.
\newblock Blocking links to minimize contamination spread in a social network.
\newblock \emph{ACM Transactions on Knowledge Discovery from Data}, 3(2): 1--23.

\bibitem[{Kingma and Ba(2014)}]{kingma2014adam}
Kingma, D.~P.; and Ba, J. 2014.
\newblock Adam: A method for stochastic optimization.
\newblock \emph{arXiv:1412.6980}.

\bibitem[{Kipf and Welling(2017)}]{kipf2016semi}
Kipf, T.~N.; and Welling, M. 2017.
\newblock Semi-supervised classification with graph convolutional networks.
\newblock In \emph{ICLR}.

\bibitem[{Ko et~al.(2020)Ko, Lee, Shin, and Park}]{ko2020monstor}
Ko, J.; Lee, K.; Shin, K.; and Park, N. 2020.
\newblock Monstor: an inductive approach for estimating and maximizing influence over unseen networks.
\newblock In \emph{ASONAM}.

\bibitem[{Leskovec, Kleinberg, and Faloutsos(2005)}]{leskovec2005graphs}
Leskovec, J.; Kleinberg, J.; and Faloutsos, C. 2005.
\newblock Graphs over time: densification laws, shrinking diameters and possible explanations.
\newblock In \emph{KDD}.

\bibitem[{Leskovec, Kleinberg, and Faloutsos(2007)}]{leskovec2007graph}
Leskovec, J.; Kleinberg, J.; and Faloutsos, C. 2007.
\newblock Graph evolution: Densification and shrinking diameters.
\newblock \emph{ACM Transactions on Knowledge Discovery from Data}, 1(1): 2--es.

\bibitem[{Leskovec and Krevl(2014)}]{snapnets}
Leskovec, J.; and Krevl, A. 2014.
\newblock {SNAP Datasets}: {Stanford} Large Network Dataset Collection.
\newblock \url{http://snap.stanford.edu/data}.

\bibitem[{Leskovec and Mcauley(2012)}]{leskovec2012learning}
Leskovec, J.; and Mcauley, J. 2012.
\newblock Learning to discover social circles in ego networks.
\newblock In \emph{NeurIPS}.

\bibitem[{Luo et~al.(2014)Luo, Cui, Zheng, and Zeng}]{luo2014time}
Luo, C.; Cui, K.; Zheng, X.; and Zeng, D. 2014.
\newblock Time critical disinformation influence minimization in online social networks.
\newblock In \emph{JISIC}.

\bibitem[{Manouchehri, Helfroush, and Danyali(2021)}]{manouchehri2021temporal}
Manouchehri, M.~A.; Helfroush, M.~S.; and Danyali, H. 2021.
\newblock Temporal rumor blocking in online social networks: A sampling-based approach.
\newblock \emph{IEEE Transactions on Systems, Man, and Cybernetics: Systems}, 52(7): 4578--4588.

\bibitem[{Ni, Zhu, and Wang(2023)}]{ni2023misinformation}
Ni, P.; Zhu, J.; and Wang, G. 2023.
\newblock Misinformation influence minimization by entity protection on multi-social networks.
\newblock \emph{Applied Intelligence}, 53(6): 6401--6420.

\bibitem[{Page et~al.(1998)Page, Brin, Motwani, and Winograd}]{page1998pagerank}
Page, L.; Brin, S.; Motwani, R.; and Winograd, T. 1998.
\newblock The pagerank citation ranking: Bring order to the web.
\newblock Technical report, Stanford University.

\bibitem[{Platto et~al.(2021)Platto, Wang, Zhou, and Carafoli}]{platto2021history}
Platto, S.; Wang, Y.; Zhou, J.; and Carafoli, E. 2021.
\newblock History of the COVID-19 pandemic: Origin, explosion, worldwide spreading.
\newblock \emph{Biochemical and biophysical research communications}, 538: 14--23.

\bibitem[{Sabottke, Suciu, and Dumitraș(2015)}]{sabottke2015vulnerability}
Sabottke, C.; Suciu, O.; and Dumitraș, T. 2015.
\newblock Vulnerability disclosure in the age of social media: Exploiting twitter for predicting Real-World exploits.
\newblock In \emph{USENIX Security}.

\bibitem[{Schneider et~al.(2011)Schneider, Mihaljev, Havlin, and Herrmann}]{schneider2011suppressing}
Schneider, C.~M.; Mihaljev, T.; Havlin, S.; and Herrmann, H.~J. 2011.
\newblock Suppressing epidemics with a limited amount of immunization units.
\newblock \emph{Physical Review E}, 84(6): 061911.

\bibitem[{Shelke and Attar(2019)}]{shelke2019source}
Shelke, S.; and Attar, V. 2019.
\newblock Source detection of rumor in social network--a review.
\newblock \emph{Online Social Networks and Media}, 9: 30--42.

\bibitem[{Tong and Du(2019)}]{tong2019beyond}
Tong, G.~A.; and Du, D.-Z. 2019.
\newblock Beyond uniform reverse sampling: A hybrid sampling technique for misinformation prevention.
\newblock In \emph{INFOCOM}.

\bibitem[{Tong et~al.(2012)Tong, Prakash, Eliassi-Rad, Faloutsos, and Faloutsos}]{tong2012gelling}
Tong, H.; Prakash, B.~A.; Eliassi-Rad, T.; Faloutsos, M.; and Faloutsos, C. 2012.
\newblock Gelling, and melting, large graphs by edge manipulation.
\newblock In \emph{CIKM}.

\bibitem[{Tripathy, Bagchi, and Mehta(2010)}]{tripathy2010study}
Tripathy, R.~M.; Bagchi, A.; and Mehta, S. 2010.
\newblock A study of rumor control strategies on social networks.
\newblock In \emph{CIKM}.

\bibitem[{Vinterbo(2002)}]{vinterbo2002note}
Vinterbo, S.~A. 2002.
\newblock A note on the hardness of the k-ambiguity problem.
\newblock \emph{Technical Report}.

\bibitem[{Wang et~al.(2013)Wang, Zhao, Chen, Li, Zhang, and Xia}]{wang2013negative}
Wang, S.; Zhao, X.; Chen, Y.; Li, Z.; Zhang, K.; and Xia, J. 2013.
\newblock Negative influence minimizing by blocking nodes in social networks.
\newblock In \emph{AAAI Workshops}.

\bibitem[{Wang et~al.(2020)Wang, Deng, Li, Yu, Jensen, and Yang}]{wang2020efficient}
Wang, X.; Deng, K.; Li, J.; Yu, J.~X.; Jensen, C.~S.; and Yang, X. 2020.
\newblock Efficient targeted influence minimization in big social networks.
\newblock \emph{World Wide Web}, 23(4): 2323--2340.

\bibitem[{Xie et~al.(2023)Xie, Zhang, Wang, Lin, and Zhang}]{xie2023minimizing}
Xie, J.; Zhang, F.; Wang, K.; Lin, X.; and Zhang, W. 2023.
\newblock Minimizing the Influence of Misinformation via Vertex Blocking.
\newblock \emph{arXiv:2302.13529}.

\bibitem[{Xu and Chen(2015)}]{xu2015scalable}
Xu, W.; and Chen, H. 2015.
\newblock Scalable rumor source detection under independent cascade model in online social networks.
\newblock In \emph{MSN}.

\bibitem[{Yan et~al.(2019)Yan, Li, Wu, Li, and Wang}]{yan2019rumor}
Yan, R.; Li, Y.; Wu, W.; Li, D.; and Wang, Y. 2019.
\newblock Rumor blocking through online link deletion on social networks.
\newblock \emph{ACM Transactions on Knowledge Discovery from Data}, 13(2): 1--26.

\bibitem[{Yang, Brenner, and Giua(2019)}]{yang2019influence}
Yang, W.; Brenner, L.; and Giua, A. 2019.
\newblock Influence maximization in independent cascade networks based on activation probability computation.
\newblock \emph{IEEE Access}, 7: 13745--13757.

\bibitem[{Yao et~al.(2014)Yao, Zhou, Xiang, Cao, and Guo}]{yao2015minimizing}
Yao, Q.; Zhou, C.; Xiang, L.; Cao, Y.; and Guo, L. 2014.
\newblock Minimizing the negative influence by blocking links in social networks.
\newblock In \emph{ISCTCS}.

\bibitem[{Yi et~al.(2022)Yi, Shan, Par{\'e}, and Johansson}]{yi2022edge}
Yi, Y.; Shan, L.; Par{\'e}, P.~E.; and Johansson, K.~H. 2022.
\newblock Edge deletion algorithms for minimizing spread in sir epidemic models.
\newblock \emph{SIAM Journal on Control and Optimization}, 60(2): S246--S273.

\bibitem[{Zareie and Sakellariou(2022)}]{zareie2022rumour}
Zareie, A.; and Sakellariou, R. 2022.
\newblock Rumour spread minimization in social networks: A source-ignorant approach.
\newblock \emph{Online Social Networks and Media}, 29: 100206.

\bibitem[{Zheng and Pan(2018)}]{zheng2018least}
Zheng, J.; and Pan, L. 2018.
\newblock Least cost rumor community blocking optimization in social networks.
\newblock In \emph{SSIC}.

\bibitem[{Zhou et~al.(2013)Zhou, Zhang, Guo, Zhu, and Guo}]{zhou2013ublf}
Zhou, C.; Zhang, P.; Guo, J.; Zhu, X.; and Guo, L. 2013.
\newblock Ublf: An upper bound based approach to discover influential nodes in social networks.
\newblock In \emph{ICDM}.

\bibitem[{Zhu et~al.(2021{\natexlab{a}})Zhu, Ni, Wang, and Li}]{zhu2021misinformation}
Zhu, J.; Ni, P.; Wang, G.; and Li, Y. 2021{\natexlab{a}}.
\newblock Misinformation influence minimization problem based on group disbanded in social networks.
\newblock \emph{Information Sciences}, 572: 1--15.

\bibitem[{Zhu et~al.(2021{\natexlab{b}})Zhu, Wang, Shi, Ji, and Cui}]{zhu2021interpreting}
Zhu, M.; Wang, X.; Shi, C.; Ji, H.; and Cui, P. 2021{\natexlab{b}}.
\newblock Interpreting and unifying graph neural networks with an optimization framework.
\newblock In \emph{WWW}.

\bibitem[{Zhu, Li, and Zhang(2016)}]{zhu2016minimum}
Zhu, Y.; Li, D.; and Zhang, Z. 2016.
\newblock Minimum cost seed set for competitive social influence.
\newblock In \emph{INFOCOM}.

\end{thebibliography}
